\documentclass{article}

\usepackage{microtype}
\usepackage{graphicx}
\usepackage{subfigure}
\usepackage{booktabs} 
\usepackage[export]{adjustbox}

\usepackage{hyperref}

\usepackage[accepted]{icml2024}

\usepackage{url}
\usepackage{hyperref} 
\usepackage{extarrows}
\definecolor{light-gray}{gray}{0.9}
\definecolor{darkgreen}{rgb}{0,0.5,0}
\definecolor{darkblue}{rgb}{0.0,0.0,0.5}
\definecolor{darkred}{rgb}{0.15,0.0,0.0}
\hypersetup{
colorlinks = true,
citecolor  = darkblue,
linkcolor  = darkred,
filecolor  = darkblue,
urlcolor   = darkblue,
}

\usepackage{amsmath,amscd,amssymb,amsthm} 
\usepackage{verbatim}
\usepackage{thmtools}
\usepackage{thm-restate}
\usepackage{nicefrac} 
\numberwithin{equation}{section}
\declaretheorem[name=Theorem, numberwithin=section]{theorem}

\declaretheorem[name=Lemma, numberlike=theorem]{lemma}
\declaretheorem[sibling=theorem]{definition}

\declaretheorem[sibling=theorem]{corollary}

\declaretheorem[ shaded={rulecolor=black, rulewidth=0.5pt, bgcolor=gray!7}, name=Theorem, sibling=theorem]{thmbox}

\declaretheorem[ shaded={rulecolor=black, rulewidth=0.5pt, bgcolor=gray!7}, name=Proposition, sibling=theorem]{propbox}

\usepackage{enumerate,enumitem}
\usepackage{mathtools}  
\mathtoolsset{showonlyrefs,showmanualtags}
\usepackage{epsfig}
\usepackage{multirow}
\usepackage{graphicx}
\usepackage{epic} 
\usepackage{eepic}
\usepackage{ifthen} 
\usepackage{amsfonts}
\usepackage{wasysym}
\usepackage{color}
\usepackage{setspace}
\usepackage{booktabs} 
\usepackage{diagbox}
\usepackage{thm-restate}
\usepackage{mdframed}
\mdfsetup{
  linecolor=black!40,
  linewidth=0.5pt
}
\setlist{nolistsep}
\setlist{leftmargin=*}

\usepackage{tikz}
\usepackage{pgfplots}
\usetikzlibrary{shapes}

\def\bz{{\mathbf z}}
\def\bx{{\mathbf x}}

\def\be{{\mathbf e}}

\def\bw{{\mathbf w}}

\def\bv{{\mathbf v}}
\def\balpha{{\boldsymbol \alpha}}

\newcommand{\sto}{\mathrm{\textsc{Grad}}}

\newcommand{\norm}[1]{\left\| #1\right\|}
\newcommand{\inp}[2]{ \langle #1,#2\rangle}
\newcommand{\rpar}[1]{\left({#1}\right)}
\newcommand{\spar}[1]{\left[{#1}\right]}
\newcommand{\abs}[1]{\left|{#1}\right|}

\renewcommand{\hat}{\widehat}
\renewcommand{\tilde}{\widetilde}
\newcommand{\variation}{\textsc{Variation}}

\newcommand{\E}{\operatorname{\mathbb{E}}}
\renewcommand{\O}{\operatorname{\mathcal{O}}}
\newcommand{\R}{\mathbb{R}}
\newcommand{\A}{\mathcal{A}}

\def\update{{\boldsymbol \Delta}}

\newcommand{\regret}{R}

\newcommand{\alg}{\textsc{Learner}}

\newcommand{\clip}{\mathrm{clip}}
\newcommand{\by}{\mathbf{y}}
\newcommand{\bg}{\mathbf{g}}
\newcommand{\bu}{\mathbf{u}}

\newcommand{\eps}{\epsilon}

\newcommand{\ind}[1]{\mathbb{I}_{[#1]}} 
\newcommand{\cc}{M_\beta}

\DeclareMathOperator*{\argmin}{arg\,min}

\icmltitlerunning{Understanding Adam Optimizer via Online Learning of Updates}

\begin{document}

\twocolumn[
\icmltitle{Understanding Adam optimizer via Online Learning of Updates:\\
Adam is FTRL in Disguise}



\icmlsetsymbol{equal}{*}

\begin{icmlauthorlist}
\icmlauthor{Kwangjun Ahn}{mit,msr}
\icmlauthor{Zhiyu Zhang}{harvard}
\icmlauthor{Yunbum Kook}{gt}
\icmlauthor{Yan Dai}{tsing} 
\end{icmlauthorlist}

\icmlaffiliation{mit}{MIT}
\icmlaffiliation{harvard}{Havard University}
\icmlaffiliation{gt}{Georgia Tech}
\icmlaffiliation{tsing}{Tsinghua University}
\icmlaffiliation{msr}{Microsoft Research}

\icmlcorrespondingauthor{Kwangjun Ahn}{kjahn@mit.edu} 

\icmlkeywords{Machine Learning, ICML}

\vskip 0.3in
]



\printAffiliationsAndNotice{}  

\begin{abstract}
Despite the success of the Adam optimizer in practice, the theoretical understanding of its algorithmic components still remains limited. In particular, most existing analyses of Adam show the convergence rate that can be simply achieved by non-adative algorithms like SGD. In this work, we provide a different perspective based on online learning that underscores the importance of Adam's algorithmic components.
Inspired by \citet{cutkosky2023optimal}, we consider the framework called \emph{online learning of updates/increments}, where we choose the updates/increments of an optimizer based on an online learner.
With this framework, the design of a good optimizer is reduced to the design of a good online learner.
Our main observation is that Adam corresponds to a principled online learning framework called Follow-the-Regularized-Leader (FTRL).
Building on this observation, we study the benefits of its algorithmic components from the online learning perspective.
\end{abstract} 
\section{Introduction}

Let $F:\R^d \to \R$ be the (training) loss function we want to minimize. In machine learning applications,  $F$ is often minimized via an iterative optimization algorithm which starts at some initialization $\bw_0$ and recursively updates 
\begin{align} \label{exp:update}
\bw_{t+1} = \bw_{t} + \update_t \quad \text{for $t=0,1\dots$}\,,
\end{align}
where $\update_t$ denotes the update/increment\footnote{Sometimes, ``update'' refers to the iterate $\bw_t$, but throughout this work, we mean the increment $\bw_{t+1}-\bw_t$.} chosen by the algorithm at the $t$-th iteration.
Practical optimizers often choose the update $\update_t$ based on the past (stochastic) gradients $\bg_{1:t} = (\bg_1,\dots,\bg_t)$  where $\bg_t$ is the stochastic gradient of $F$ collected during the $t$-th iteration.
For instance, stochastic gradient descent (SGD) corresponds to choosing $\update_t = -\alpha_t \bg_t$ in \eqref{exp:update} for some learning rate $\alpha_t>0$.

For training deep neural networks, one of the most popular choices is the Adam optimizer \citep{kingma2015adam}.
In particular, several recent works have observed that Adam and its variants are particularly effective for training Transformer-based neural network models \citep{zhang2020adaptive,kunstner2023noise,jiang2022does,pan2023toward,ahn2023linear}.
Given some learning rate $\gamma_t>0$ and discounting factors $\beta_1,\beta_2\in (0,1)$, Adam chooses $\update_t$ on each  coordinate $i=1,2,\dots, d$  by combining $\bg_{1:t}$ as\footnote{For simplicity, we remove the debiasing step and the appearance of $\eps$ in the denominator used in the original paper.}
\[ \update_{t}[i] = -\gamma_t \frac{(1-\beta_1) \sum_{s=1}^t\beta_1^{t-s}\bg_s[i]}{ \sqrt{  (1-\beta_2^2)\sum_{s=1}^t (\beta_2^{t-s}\bg_s[i])^2}}\,,
\]
where $\bv[i]$ denotes the $i$-th coordinate of a vector $\bv$.
For a streamlined notation, we define the \emph{scaled learning rate} $\alpha_t \gets \gamma_t\cdot \nicefrac{(1-\beta_1)}{\sqrt{1-\beta_2^2}}$ and consider 
\[  \tag{Adam} \label{adam}
\update_{t}[i] = -\alpha_t \frac{ \sum_{s=1}^t\beta_1^{t-s}\bg_s[i]}{ \sqrt{ \sum_{s=1}^t (\beta_2^{t-s}\bg_s[i])^2}}\,.
\]
Compared to SGD, the notable components of Adam is the fact that it aggregates the past gradients $\bg_{1:t}$ (\emph{i.e.}, \textbf{momentum}) with the \textbf{discounting factors} $\beta_1,\beta_2$.

Despite the prevalent application of Adam in deep learning, our theoretical grasp of its mechanics remains incomplete, particularly regarding the roles and significance of its core elements: the \textbf{momentum} and the \textbf{discounting factors}. 
Most existing theoretical works on Adam and its variants primarily focus on characterizing the convergence rate for convex functions or smooth nonconvex functions~\citep{reddi2018on,zhou2018adashift,chen2019convergence,zou2019sufficient,alacaoglu2020new,guo2021novel,defossez2022simple,zhang2022adam,li2023convergence,wang2023closing} for which methods like SGD already achieve the minimax optimal convergence rate. 
In fact, the latest works in this line~\citep{li2023convergence,wang2023closing} both mention that their convergence rate of Adam gets worse with momentum~\citep[\S 6]{wang2023closing} or the rate of Adam is no better than that of SGD~\citep[\S 7]{li2023convergence}.
A notable exception is
\citet{crawshaw2022robustness} where they show the benefits of momentum in a variant of Adam, under the generalized smoothness conditions of \citet{zhang2020why}.

In this work, we take a different approach to understand Adam from a online learning perspective, as outlined below.

\subsection{Our Approach and Main Results}
Our starting point is the main insight of \citet{cutkosky2023optimal} that the design of nonconvex optimizers falls under the scope of \emph{online linear optimization}, an iconic setting in online learning. 
Specifically, one can regard the selection of the update $\update_t$ based on $\bg_{1:t}$ as an online prediction procedure.
Such a framework will be called \textbf{online learning of updates/increments} \eqref{olu}.

Building on this framework, we then notice that it is important to choose an online learner that performs well in dynamic environments~\citep{cutkosky2023optimal}. Better dynamic regret leads to better optimization performance (\autoref{thm:informal}), and therefore, the design of good optimizers is reduced to designing good \emph{dynamic online learners}. Along this line, our results can be summarized as follows:
\begin{itemize}[leftmargin=*,itemsep=0pt,topsep=0pt]
\item (\autoref{sec:olu}:) Our main observation is that the popular Adam optimizer corresponds to choosing a classical online learner called Follow-the-Regularized-Leader (FTRL) \citep{gordon1999regret,kalai2005efficient,shalev2006online,abernethy2008competing,nesterov2009primal,hazan2008extracting}.
Specifically, when using the framework \ref{olu}, Adam is recovered by plugging in a discounted instance of FTRL well-suited for dynamic environment, which we call \ref{dftrl}.

\item (\autoref{sec:dynamic}:) We provide the dynamic regret guarantees of \ref{dftrl} (\autoref{thm:dftrl} and \autoref{thm:dftrl-clip}) through a novel \textbf{discounted-to-dynamic conversion}. It gives us a new perspective on the role of Adam's algorithmic components, namely the momentum  and the discounting factors.  Our results suggest that \textbf{both components are crucial for designing a good dynamic online learner} (see \autoref{sec:components}).

\item (\autoref{sec:optimization}:)  We justify the importance of a good dynamic regret, via its implications for optimization. 
Along the way, we discuss  optimization settings for which Adam could be potentially beneficial. 

\end{itemize}

\section{Adam is FTRL in Disguise}
\label{sec:olu}

Iterative optimization algorithms are closely connected to adversarial online learning. For example, SGD is often analyzed through online gradient descent (OGD), its online learning counterpart.  
To exploit this connection in \eqref{exp:update}, the traditional approach is using an online learner to directly choose the \emph{iterates} $\bw_t$, as demonstrated by \citet{bottou1998online,cesa2004generalization,duchi10adagrad,li2019convergence,ward2019adagrad} and many more. Diverging from this common approach, we consider a new approach due to  \citet{cutkosky2023optimal} that applies the online learner to choose the \textit{updates} $\update_t$.

\subsection{Choosing Updates/Increments via Online Learning}
\label{sec:ol}
Consider an iconic setting of online learning called \emph{online linear optimization} (OLO). For the consistency with our optimization algorithm \eqref{exp:update}, we will introduce OLO using slightly nonstandard notations. In each round $t$, the algorithm  (or online learner) chooses a point $\update_t\in\R^d$, and then receives a linear loss function $\ell_t(\cdot)= \langle \bv_{t+1},\cdot\rangle$ and suffers the loss of $\ell_t(\update_t)$. 
In other words, it chooses $\update_t$ based on the previous loss sequence $\bv_{1:t}\coloneqq (\bv_1,\bv_2,\dots, \bv_t)$ and then receives the next loss $\bv_{t+1}$.
The performance of the online learner is measured by the \emph{regret} against a comparator sequence $\bu_{0:T-1}$, defined as 
\begin{align} \label{exp:dynamic_regret}
\regret_{T}(\bu_{0:T-1}) \coloneqq    \sum_{t=1}^{T} \inp{\bv_{t}}{\update_{t-1}-\bu_{t-1}}\,.
\end{align}
To be precise, \eqref{exp:dynamic_regret} is called the \emph{dynamic regret} in the literature~\citep{zinkevich2003online}. Another common metric, \emph{static regret}, is a special case of \eqref{exp:dynamic_regret} where all $\bu_t=\bu$; this is denoted as $\regret_{T}(\bu)$.

Now given an online learner \alg{}, we consider an optimization algorithm that outputs the $\update_t$ in \eqref{exp:update} using \alg{}. More formally, 
\begin{align} \tag{OLU}\label{olu}
\boxed{\text{$\update_t$ is chosen by \alg{} based on $\bg_{1:t}$.}}
\end{align}
We call this framework \emph{online learning of updates} (or \emph{online learning of increments}).
This framework was first proposed by \citet{cutkosky2023optimal} (under the name \emph{online-to-nonconvex conversion}) to design algorithms that find critical points for nonsmooth and nonconvex stochastic optimization problems.
Under \ref{olu}, we want \alg{} of choice to be a good online learner for dynamic environments, as summarized in the informal statement below.

\begin{theorem}[{\bf Importance of dynamic regret in \ref{olu}}; see \autoref{thm:guarantee}] \label{thm:informal}
In \ref{olu}, a better dynamic regret of \alg{} leads to a better optimization guarantee. Therefore, we want \alg{} to have a low dynamic regret.
\end{theorem}

To understand this elegant reduction,
let us give examples of how \autoref{thm:guarantee} is applied.
Recent works \citep{cutkosky2023optimal,zhang2024random} choose an online gradient descent (OGD)~\cite{zinkevich2003online} as \alg{} to design algorithms for finding stationary points for nonconvex and nonsmooth functions.
When \alg{} is chosen as OGD, the resulting optimization algorithm under \ref{olu} turns out to be SGD with momentum \citep{zhang2024random}.
However, OGD is known to require a careful tuning of learning rate~\cite{zinkevich2003online}.
What if we use an adaptive online learner as \alg{}?

Our main observation is that Adam can be recovered by choosing \alg{} as an adaptive version of Follow-the-Regularized-Leader that is well-suited for dynamic environments, which we gradually elaborate. 

\subsection{Basics of Follow-the-Regularized-Leader (FTRL)}\label{sec:ftrl_basics}

Follow-the-Regularized-Leader (FTRL) is a classical algorithmic framework in online learning. Unlike the more intuitive descent-type algorithms, the key idea of FTRL is selecting the decisions by solving a convex optimization problem in each round.
Throughout, we focus on the 1D case of OLO ($d=1$) since the update of Adam is coordinate-wise.
In particular, the overall regret for the $d$-dimension would be the sum of the regret of each coordinate.

The 1D linear loss function is given by $\ell_t(\Delta) = v_{t+1}\Delta$ for $v_{t+1}\in \R$. We use the subscript $t+1$ to highlight that it is only revealed after deciding $\Delta_t$.

FTRL relies on a nonnegative convex \emph{regularizer} $\Phi$, which is set to $\Phi=\frac{1}{2}|\cdot|^2$ in this work.
The algorithm initializes at $\Delta_0 =0$ and in each round outputs
\begin{align}\tag{FTRL} \label{ftrl}
\Delta_t =  \argmin_{x}\Bigl[  
\frac{1}{\eta_t}\Phi(x) + \sum_{s=1}^{t} v_{s} x\Bigr]  = -\eta_t \sum_{s=1}^{t} v_{s} \,,
\end{align} 
where the effective step size $\eta_t>0$ is non-increasing in $t$. The remaining task is to choose good step sizes $\eta_t$.

One prominent choice is the \emph{adaptive} step size of the   \textbf{scale-free FTRL} algorithm \citep[\S 3]{orabona2018scale}  in the style of~\citet{mcmahan2010adaptive,duchi10adagrad}.
Scale-free FTRL chooses $\eta_t= \nicefrac{\alpha}{ \sqrt{ \sum_{s=1}^t v_s^2}}$ based on a scaling factor $\alpha>0$,  resulting in the update
\begin{equation}\label{soloftrl}
\Delta_t= -\alpha \frac{\sum_{s=1}^{t} v_{s}}{ \sqrt{ \sum_{s=1}^t v_s^2}}\,.
\end{equation}
Here if the denominator is zero, then we set the output $\Delta_t=0$.
The update 
\eqref{soloftrl} is independent of any constant scaling of loss sequence $v_{1:t}$, making it a \emph{scale-free} update.
This is beneficial when the magnitude of loss sequence varies across different coordinates.

We remark that the analysis of scale-free FTRL is in fact quite subtle, as echoed by \citet{mcmahan2017survey,orabona2018scale}. Using a different proof strategy, we prove a static regret bound (\autoref{thm:adagrad-ftrl}) of scale-free FTRL that slightly strengthens that of \citet{orabona2018scale}.  

\subsection{Adam Corresponds to Discounted-FTRL}

Now back to \ref{olu}, let us use \ref{ftrl} to choose the update $\Delta_t$. Denoting the coordinate-wise gradients in optimization by $g_{1:t}$, a na\"ive approach is to use \ref{ftrl} directly  by setting $v_t\gets g_t$.
Unfortunately, this approach is not a good one because \ref{ftrl} is designed to achieve low static regret, while   \ref{olu} requires low  \emph{dynamic} regret.
In fact, it is shown by \citep[Theorem 2]{jacobsen2022parameter} that  Algorithms of the form \eqref{ftrl} are \textbf{not} good dynamic online learners. See also the lower bounds in \autoref{thm:lower}.

One could already see intuitively why this is the case: any algorithm in this form does not ``forget the past'', as the output $\Delta_t$ is  the (regularized) minimizer of the cumulative loss $L_t(x)\coloneqq \sum_{s=1}^t v_s x$. Therefore, it is only competitive \emph{w.r.t.} a fixed comparator that minimizes $L_t(x)$, instead of a time-varying comparator sequence. 

To address this issue, our approach is to ``\textbf{discount}'' the losses from the distant past. In particular, we implement this by gradually up-scaling the losses over time. The intuition is that when deciding the output $\Delta_t$, the recent losses would have much higher ``weights'' compared to older ones, which essentially makes the latter negligible. 
\begin{theorem}[Informal; see Theorems~\ref{thm:dftrl} and \ref{thm:dftrl-clip}] \label{thm:informal_dynamic}
For some $\beta\in(0,1)$, the discounted version of scale-free FTRL that internally replaces $v_t$ by $\beta^{-t} v_t$ is a good dynamic online learner.
\end{theorem}
\noindent Remarkably, plugging this discounted scale-free FTRL into \ref{olu} would almost recover Adam. There are just two small issues: in the \ref{adam} update, the (scaled) learning rate $\alpha_t$ is time-varying, and we need two discounting factors $\beta_1$ and $\beta_2$ for the numerator and the denominator separately. It is not hard to fix this last bit, and we end up with an \ref{ftrl} instance which given the input $g_t$ picks
\begin{align}\label{eq:discounted_ftrl}
v_t\gets \beta_1^{-t} g_{t}\,,\quad \textrm{and}\quad\eta_t =\frac{\alpha_t (\beta_1/\beta_2)^t}{\sqrt{  \sum_{s=1}^t (\beta_2^{-s}g_s)^2}}\,.
\end{align}

Collecting all the pieces above yields our first main result. 

\begin{propbox}[\textbf{Adam is discounted-FTRL in disguise}] \label{prop:adam}
For some learning rate $\alpha_t>0$ and discounting factors $\beta_1,\beta_2 \in (0,1]$, \ref{ftrl} with \eqref{eq:discounted_ftrl} is equivalent to picking
\begin{align} \label{dftrl_with_beta12}
\Delta_t = -\alpha_t \frac{ \sum_{s=1}^t\beta_1^{t-s}g_s}{ \sqrt{  \sum_{s=1}^t (\beta_2^{t-s}g_s)^2}}\,.
\end{align}
Applying it as a coordinate-wise \alg{} in \ref{olu} recovers the Adam optimizer in \ref{adam}.
\end{propbox}

Recall that \ref{olu} connects the problem of optimization to the well-established problem of dynamic regret minimization.
Given \autoref{prop:adam}, we make use of this connection to understand the components of Adam from the dynamic regret perspective.
That is the main focus of the next section.
Before getting into that, we briefly compare our approach with existing derivations of Adam based on FTRL.

\subsection{Comparison with the Previous Approach}

In fact, \citet{zheng2017follow} propose a derivation of Adam based on FTRL.
However, their approach is quite different than ours, as we detail below.

We first briefly summarize the approach of \citet{zheng2017follow}. Their main idea is to consider the ``weighted'' version of proximal-FTRL defined as
\[
    w_t = \argmin_{w}  \sum_{s=1}^{t}  \lambda_s \Bigl( \left\langle g_s, w \right\rangle + \frac{1}{2} \lVert w - w_{s-1} \rVert^2_{Q_s} \Bigr)\,,
\]
for some weights $\{\lambda_s\}$ and positive semi-definite matrices $\{Q_s\}$. 
Given this, their main observation is that Adam roughly corresponds to this proximal-FTRL with carefully chosen $\{\lambda_s\}$ and $\{Q_s\}$.

Although their motivation to explain Adam with a version of FTRL is similar to ours, we highlight that their approach is different than ours.
In fact, our approach overcomes some of the limitations of \citet{zheng2017follow}.  
\begin{itemize}[leftmargin=*,itemsep=0pt,topsep=0pt]
    \item Firstly, their derivation actually needs a heuristic adjustment of changing the anchor points of the regularizer from $w_{s-1}$ to $w_{t-1}$. A priori, it is not clear why such adjustment is needed, and to the best of our knowledge, there is no formal justification given. But with our approach, such an adjustment is naturally derived because under OLU, the online learner chooses the update/increment instead of the iterate.  
    \item Secondly, in \citet{zheng2017follow}, in order to recover Adam, they have to choose $\{w_s\}$ and $\{Q_s\}$ carefully, which also lacks justification. One of the main advantages of our approach is the fact that the discounting factors are theoretically justified via the dynamic regret perspective. More specifically, we show that without the discounting factor, FTRL is not a good dynamic learner.  
\end{itemize}

\section{Discounted-FTRL as a Dynamic Learner}
\label{sec:dynamic}

This section provides details on \autoref{thm:informal_dynamic}, focusing on the special case of \ref{adam} where $\beta_1,\beta_2=\beta$ for some $\beta\in (0,1]$, and $\alpha_t =\alpha$ for some $\alpha>0$. From \autoref{prop:adam} and using the same notation as \eqref{soloftrl}, this corresponds to the following coordinate-wise update rule:
\begin{align} \tag{$\beta$-FTRL} \label{dftrl}
\Delta_t = -\alpha \frac{ \sum_{s=1}^t\beta^{t-s}v_s}{ \sqrt{   \sum_{s=1}^t (\beta^{t-s}v_s)^2}}\,,
\end{align}
and if the denominator is zero, we define $\Delta_t=0$. We call this algorithm \ref{dftrl}. With $\beta=1$, it exactly recovers scale-free FTRL \eqref{soloftrl} which is shown to be a poor dynamic learner \citep{jacobsen2022parameter}. Therefore, we will  focus on $\beta<1$ in the dynamic regret analysis.

The earlier informal result (\autoref{thm:informal_dynamic}) is formalized in \autoref{thm:dftrl} (for unbounded domain) and \autoref{thm:dftrl-clip} (for bounded domain). We provide the simplified versions here, deferring the detailed adaptive version to \autoref{thm:dftrl_formal}.

\begin{thmbox}[{\bf Dynamic regret of \ref{dftrl}; unbounded domain}] \label{thm:dftrl}
For a loss sequence $v_{1:T}$,
consider \ref{dftrl} with $\beta < 1$ and some constant $\alpha>0$.
Let  
\begin{align}
\cc \coloneqq \max_{t\in[1,T]} \frac{ \abs{\sum_{s=1}^t\beta^{t-s}v_s}}{ \sqrt{   \sum_{s=1}^t (\beta^{t-s}v_s)^2}}\,.
\end{align} 
Then, for any comparator sequence $u_{0:T-1}$ such that $ \abs{u_t}\leq \alpha \cc$ for all $t$,  the dynamic regret $\regret_{T}(u_{0:T-1})$ is upper bounded by
\begin{align}
\O\Bigg(\frac{\bigl( \alpha \cc^2+\cc P \bigr)G}{\sqrt{1-\beta}}+\sqrt{1-\beta}\cdot  \alpha\cc^2GT\Bigg)\,.
\end{align} 
Here, $G\coloneqq\max_{t\in[1:T]}\abs{v_t}$,  and $P\coloneqq\sum_{t=1}^{T-1}\abs{u_{t}-u_{t-1}}$ is the \emph{path length}.
\end{thmbox}
We sketch the proof of \autoref{thm:dftrl} in \autoref{sec:sketch}.

\autoref{thm:dftrl} may not seem straightforward, so let us start with a high level interpretation. 
First of all, the path length $P$ is a standard complexity measure of the comparator $u_{0:T-1}$ in the literature~\citep{herbster2001tracking}, which we examine closely in \autoref{sec:dynamic_basics}.
In the context of dynamic online learning, the above bound could be reminiscent of a classical result from \citep[Theorem~2]{zinkevich2003online}: on a domain of diameter $D$, the dynamic regret of \emph{online gradient descent} (OGD) with learning rate $\eta$ can be bounded as
\begin{align}\label{eq:ogd_dynamic}
\regret_{T}(u_{0:T-1})\leq   \O\rpar{\frac{D^2+DP}{\eta}+\eta G^2T}\,.
\end{align}
Intuitively, the choice of $\eta$ balances the two conflicting terms on the RHS, and a similar tradeoff remains as a recurring theme in the dynamic online learning literature \citep{hall2015online,zhang2018adaptive,jacobsen2022parameter}. In an analogous manner, the discounting factor $\beta$ in  \autoref{thm:dftrl} largely serves the similar purpose of balancing conflicting factors. A rigorous discussion is deferred to \autoref{sec:components}.

As a complementary result to \autoref{thm:dftrl}, we also present a dynamic regret bound for the case of \emph{a priori bounded domain}, where the outputs of online learner should lie in a bounded domain $[-D,D]$. 
In this case, we project the output of \ref{dftrl} to the given domain:
\begin{align} \tag{$\beta$-FTRL$_{D}$} \label{dftrl-clip}
\Delta_t = -\clip_{D}\rpar{\alpha \frac{ \sum_{s=1}^t\beta^{t-s}v_s}{ \sqrt{   \sum_{s=1}^t (\beta^{t-s}v_s)^2}}}\,,
\end{align}
where $\clip_D(x):=x\min(\frac{D}{|x |},1)$.
Then, with the same notations of $G$ and $P$ as in \autoref{thm:dftrl}, we get the following result (see \autoref{sec:dftrl-clip} for details).

\begin{thmbox}[{\bf Dynamic regret of \ref{dftrl-clip}; bounded domain}] \label{thm:dftrl-clip}
For $D>0$, consider any comparator sequence $u_{0:T-1}$ such that $\abs{u_t}\leq D$ for all $t$.
Then for any loss sequence $v_{1:T}$, \ref{dftrl-clip} with $\beta<1$ and $\alpha=D$ has the dynamic regret $\regret_{T}(u_{0:T-1})$ upper bounded by
\begin{align}
\O\rpar{  \frac{DG}{\sqrt{1-\beta}}+\frac{GP}{1-\beta} + \sqrt{1-\beta} DGT}\,.
\end{align} 
\end{thmbox}

Compared to \autoref{thm:dftrl}, the main difference is that the regret bound now holds simultaneously for all the loss sequences $v_{1:T}$ of \emph{arbitrary} size. The price to pay is the requirement of knowing $D$, and the multiplying factor on the path length $P$ is slightly worse, \emph{i.e.}, $(1-\beta)^{-1/2}\rightarrow (1-\beta)^{-1}$. A sneak peek into the details: such a slightly worse factor is due to the projection step breaking the \emph{self-bounding} property of \ref{dftrl}, which says the discounted gradient sum $\sum_{s=1}^t\beta^{t-s}v_s$ can be controlled by the maximum update magnitude $\sup \abs{\Delta_t}$ times the empirical variance of gradients, \textit{i.e.}, $\sqrt{\sum_{s=1}^t(\beta^{t-s}v_s)^2}$. Interested readers may compare Subsections \ref{sec:simplify} and \ref{sec:dftrl-clip} for the subtleties.

Moving forward, Theorems \ref{thm:dftrl} and \ref{thm:dftrl-clip} constitute our main results characterizing the dynamic regret of \ref{dftrl}. However, there is still one  missing piece. The earlier informal result  (\autoref{thm:informal_dynamic}) states that
\begin{center}
\textit{\ref{dftrl} is a ``good'' dynamic online learner}.
\end{center}

However, we have never explained which dynamic regret is good. Actually, the ``goodness'' criterion in dynamic online learning could be a bit subtle, as the typical sublinear-in-$T$ metric in static online learning becomes vacuous. Next, we briefly provide this important background.  

\subsection{Basics of Dynamic Online Learning}\label{sec:dynamic_basics}

Dynamic online learning is intrinsically challenging. 
It is well-known that regardless of the algorithm, there exist loss and comparator sequences such that the dynamic regret is at least $\Omega(T)$. This is in stark contrast to static regret bounds in OLO, where the standard minimax optimal rate is the sublinear in $T$, \emph{e.g.}, $\O\bigl(\sqrt{T}\bigr)$. 

To bypass this issue, the typical approach is through \emph{instance adaptivity}. Each combination of the loss and comparator sequences can be associated to a \emph{complexity measure}; the larger it is, the harder regret minimization becomes. Although it is impossible to guarantee sublinear-in-$T$ regret bounds against the hardest problem instance, one can indeed guarantee a regret bound that \emph{depends on} such a complexity measure. From this perspective, the study of dynamic online learning centers around finding suitable complexity measures and designing adaptive algorithms.

Only considering the comparator sequence $\bu_{0:T-1}$, the predominant complexity measure is the \emph{path length} $P\coloneqq\sum_{t=1}^{T-1}\norm{\bu_{t}-\bu_{t-1}}$ \citep{zinkevich2003online}, whose 1D special case is considered in \autoref{thm:dftrl} and \autoref{thm:dftrl-clip}. On a bounded domain with diameter $D$, the optimal dynamic regret bound is $\O(G\sqrt{DPT})$, which can be achieved through the classical result for OGD \eqref{eq:ogd_dynamic} with the \emph{$P$-dependent} learning rate $\eta=G^{-1}\sqrt{DP/T}$. On top of that, one could use a model selection approach \citep{zhang2018adaptive} to avoid the infeasible \emph{oracle tuning} (\textit{i.e.}, $\eta$ depends on the unknown $P$), at the expense of increased computation. Recent works \citep{jacobsen2022parameter,zhang2023unconstrained} further extend such results to unbounded domains. 

The essential ``goodness'' of this $\O\bigl(G\sqrt{DPT}\bigr)$ bound is due to $P\leq DT$. In the worst case the bound is trivially $\O(DGT)$, but if the comparator is \emph{easy} (\emph{i.e.}, $P=\O(D)$), then it becomes $\O\bigl(DG\sqrt{T}\bigr)$, recovering the well-known optimal static regret bound. In general, the goodness of a dynamic regret bound is usually measured by the exponents of both $P$ and $T$ (\textit{e.g.}, $\frac 12$ and $\frac 12$ in $\O \bigl (G\sqrt{DPT} \bigr )$).  

Given this background, we now use the dynamic regret results of \ref{dftrl} so far to interpret the role of two key components of Adam, namely the \textbf{momentum} (\emph{i.e.}, aggregating past gradients) and the \textbf{discounting factor $\beta$} (\emph{i.e.}, exponential moving average).

\subsection{Benefits of Momentum and Discounting Factor}
\label{sec:components}

Recall that a particular strength of the \ref{olu} framework is that it establishes \textbf{one-to-one correspondence} between optimizers and their online learning counterparts. 
Thus we can compare a variety of optimizers by \textbf{comparing their corresponding online learners}, from the perspective of dynamic regret. 
Notice that \ref{dftrl} from our analysis corresponds to the scaled parameterization of \eqref{adam}. Through that, our ultimate goal is to shed light on Adam's algorithmic components --- the momentum and the discounting factor.

We first discuss the baseline online learners  for this problem:
\begin{itemize}[leftmargin=*,itemsep=0pt,topsep=0pt]
\item To understand the momentum, we pick the baselines as a family of ``degenerate'' online learners that induce \emph{non-momentum optimizers}, such as SGD and AdaGrad \citep{duchi10adagrad}. 
Concretely, for the loss sequence $\bv_{1:T}$, this family of \alg{} in \ref{olu} has the following generic update rule: for some coordinate-wise learning rate $\balpha_t[i]>0$ $\forall i$, it outputs
\begin{align}  
\update_{t}[i] &= - \balpha_t[i] \bv_t[i]\,.   \label{no-momentum}  
\end{align}
For example, given a scalar $\alpha_t>0$, SGD chooses the coordinate-wise learning rate $\balpha_t[i]$ independently of the coordinates, \emph{i.e.},
\begin{align}
\balpha_t[i] &= \alpha_t\,, \tag{SGD}\label{sgd}
\end{align}
while AdaGrad further employs a variance-based preconditioning, \emph{i.e.},
\begin{align}
\balpha_{t}[i] &= \frac{\alpha_t}{{\sqrt{\sum_{s=1}^{t} \bv_s[i]^2}}}\,.   \tag{AdaGrad} \label{adagrad}
\end{align}
The important observation is that compared to the \eqref{ftrl} update rule, \textbf{the coordinate-wise update $\Delta_t[i]$ in \eqref{no-momentum} only scales linearly with the most recent observation $\bv_t[i]$}, instead of using the entire history $\bv_{1:t}[i]$. In other words, from the optimization perspective, this family of algorithms does not make use of the past history of gradients to decide the update direction.

\item To understand the discounting factor,  we pick the baseline as \ref{dftrl} with $\beta=1$ (\emph{i.e.}, no discounting). Alternatively, if the domain is bounded, then we use the clipped version of \ref{dftrl-clip} with $\beta=1$ instead. In other words, such baselines correspond to scale-free FTRL \eqref{soloftrl}. 
In light of \citet{jacobsen2022parameter}, the case of $\beta=1$ is not a good dynamic online learner, which we discuss more formally below.
\end{itemize}
The following lower bound, inspired by \citet{jacobsen2022parameter}, shows that the above baselines fail to achieve sublinear dynamic regret for a very benign example of $P=\O(1)$.
See \autoref{pf:lower} for a proof.
\begin{theorem}[Lower bounds for baselines] \label{thm:lower}
Consider a 2D online linear optimization problem with the bounded domain $[-1,1]^2$.
 For any given $T$,  there exist ($i$) a loss sequence  $\bv_1,\dots,\bv_T \in \R^2$ with $\|\bv_t\|=1$ for all $t$, and ($ii$) a comparator sequence $\bu_0,\dots,\bu_{T-1}\in[-1,1]^2$ with the coordinate-wise path length $\sum_{t=1}^{T-1} |\bu_t[i]-\bu_{t-1}[i] |\le 1$ for both $i=1,2$, such that the following holds:
 \begin{itemize}[leftmargin=*,itemsep=0pt,topsep=0pt]
 \item For all $t$, $\bu_{t-1}\in\argmin_{\bu\in[-1,1]^2}\inp{\bv_t}{\bu}$.
     \item  Any ``non-momentum'' online learner of the form \eqref{no-momentum} has the dynamic regret at least $T-3$.
     \item \ref{dftrl-clip} with $\beta=1$ and $D=1$ has the dynamic regret at least $(T-3)/2$.  
 \end{itemize} 
\end{theorem}

The first bullet point says that the constructed comparator sequence $\bu_{0:T-1}$ is the best ones \emph{w.r.t.} the loss sequence $\bv_{1:T}$, therefore the dynamic regret against such $\bu_{0:T-1}$ is a good metric to measure the strength of dynamic online learners. 
Then, the rest of the theorem shows that the two baselines above (corresponding to optimizers without momentum or discounting factor under \ref{olu}) cannot guarantee low regret against $\bu_{0:T-1}$, thus are not good dynamic online learners.

In contrast, \autoref{thm:dftrl-clip} shows that   \ref{dftrl-clip} with $\beta<1$ is a better dynamic online learner, and we make this very concrete through the following corollary. Again, it suffices to consider the 1D setting.

\begin{corollary} \label{cor:discount_bound}
For $D>0$, consider any comparator sequence $u_{0:T-1}$ such that $\abs{u_t}\leq D$ for all $t$.
Then, given any constant $c>0$, \ref{dftrl-clip} with $\beta = 1 - c T^{-\nicefrac{2}{3}}>0$ achieves the dynamic regret bound
\begin{align}
\regret_{T}(u_{0:T-1})\leq \O\left( DGT^{\nicefrac{2}{3}} c^{\nicefrac 1 2}( 1+ c^{-\nicefrac 3 2}P/D)\right)\,,
\end{align} 
which enjoys a $T^{\nicefrac{2}{3}}$ dependency on $T$. In particular, with the optimal tuning $c= \Theta\rpar{(P/D)^{\nicefrac{2}{3}}}$, the bound becomes $\O(D^{\nicefrac{2}{3}}GP^{\nicefrac{1}{3}}T^{\nicefrac{2}{3}})$. 
\end{corollary} 

The proof is deferred to \autoref{pf:cor:discount}. We emphasize that in the lower bound example of \autoref{thm:lower}, we have $P =\O(D)$, so the $\beta<1$ case achieves a sublinear dynamic regret bound of $\O(DGT^{\nicefrac{2}{3}})$. This suggests that in order to design a better dynamic online learner \textbf{both momentum and discounting factor are necessary}.

A similar result can be developed for the case of unbounded domain under an assumption regarding the 1D OLO environment that generates the losses $v_{1:T}$.

\begin{corollary}\label{cor:discount_unbound}
Assume the environment is well-behaved in the sense that $\cc \leq M$ for all $\beta\in (0,1]$. 
Consider any comparator sequence $u_{0:T-1}$ such that $\abs{u_t}\leq \alpha M$ for all $t$. Then, given any constant $c>0$, \ref{dftrl} with parameters $\alpha$ and $\beta = 1- cT^{-1}>0$ achieves the dynamic regret bound
\begin{align}
\regret_{T}(u_{0:T-1})\leq   \O\rpar{ \alpha M^2 G \sqrt{T} c^{\nicefrac 1 2}\rpar{1+\frac{c^{-1}P}{\alpha M}}}\,,
\end{align}
which enjoys a $\sqrt{T}$ dependency on $T$. In particular, with the optimal tuning $c = \Theta\bigl(P/ (\alpha M)\bigr)$, the bound becomes $\mathcal{O}\bigl(\alpha^{\nicefrac{1}{2}}M^{\nicefrac{3}{2}}G\sqrt{PT}\bigr)$. 

In contrast, \ref{dftrl} with the same $\alpha$ but the different $\beta=1$ achieves a dynamic regret bound which is linear in $P$:
\begin{align}
\regret_{T}(u_{0:T-1})\leq   \O\rpar{ M  GP\sqrt{T}}\,.
\end{align}
\end{corollary}

Essentially, without the discounting factor we can show an $\O\bigl(P\sqrt{T}\bigr)$ dynamic regret bound, but with discounting the bound can be improved to the optimal rate $\O\bigl(\sqrt{PT}\bigr)$, under suitable tuning.  This provides another evidence that discounting is helpful for designing a dynamic online learner.  
We remark that without discounting, the dynamic regret of $\O\bigl(P\sqrt{T}\bigr)$ is  unimprovable in light of the lower bound result~\citep[Theorem 3]{jacobsen2022parameter}.

\subsection{Proof Sketch of \autoref{thm:dftrl}}
\label{sec:sketch}
Finally, we briefly sketch the proof of \autoref{thm:dftrl}, the dynamic regret bound of \ref{dftrl}. The proof of \autoref{thm:dftrl-clip} mostly follows the same analysis. 
Our analysis of the dynamic regret relies on the following  ``discounted''  regret. 
\begin{definition} [{$\beta$-discounted regret}]\label{def:discounted_regret}
For any discounting factor $\beta\in(0,1]$, the $\beta$-discounted regret is defined as
\[
\regret_{T;\beta}(u)\coloneqq \sum_{t=1}^{T} \beta^{T-t}v_{t}(\Delta_{t-1} -u)\,.
\]
\end{definition}
It is noteworthy that  the discounted regret has been considered in the concurrent works \citep{zhang2024discounted,jacobsen2024online} to adapt online learners to dynamic environments.
Moreover, the discounted regret has found to be useful in designing nonconvex optimization algorithms, as shown in \citep{zhang2024random,ahn2024adam}.
In our work, we propose a generic conversion approach to analyzing the dynamic regret using the discounted regret, called \textbf{discounted-to-dynamic conversion} (\autoref{thm:discount_to_dynamic}), which could be of independent interest. 
At a high level, our analysis follows the following steps.
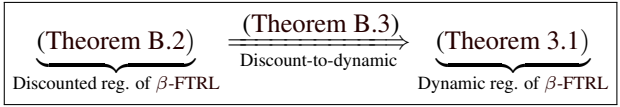
\begin{align}
\boxed{\underbrace{\text{(\autoref{thm:dftrl_discount_regret}})}_{\text{\scriptsize{Discounted reg. of \ref{dftrl}}} }     \xLongrightarrow[\text{Discount-to-dynamic}]{\text{ \normalsize (\autoref{thm:discount_to_dynamic})}}  \underbrace{\text{(\autoref{thm:dftrl}})}_{\text{\scriptsize{Dynamic reg. of \ref{dftrl}}} } }
\end{align}
\begin{enumerate}[leftmargin=*,itemsep=0pt,topsep=0pt]
\item \ref{dftrl} is the discounted version of scale-free FTRL, and the latter is associated to a static regret bound (\autoref{thm:adagrad-ftrl}). Utilizing this relation, we naturally arrive at a  discounted regret bound of \ref{dftrl} (\autoref{thm:dftrl_discount_regret}). 
Intuitively, it measures the performance of \ref{dftrl} on an exponentially weighted, ``local'' look-back window that ends at time $T$. A particular strength is that the bound is \emph{anytime}, \emph{i.e.}, it holds for all $T$ simultaneously.

\item Next, consider the dynamic regret over the entire time horizon $[1,T]$. We can imagine partitioning $[1,T]$ into concatenating subintervals, and on each of them the dynamic regret can be approximately upper-bounded by suitable \emph{aggregations} of the above discounted regret bound (modulo certain approximation error) --- this is because the discounted regret bound mostly concerns the few recent rounds, so the dynamic regret can be controlled by the sum of such local metrics.
Formalizing this argument results in the  discounted-to-dynamic conversion in \autoref{thm:discount_to_dynamic}: the dynamic regret of an algorithm is expressed using its discounted regret  as an equality.
\end{enumerate}

Both \autoref{thm:dftrl} and~\ref{thm:dftrl-clip} are obtained by combining these two elements, with slightly different ways of relaxation. 

\section{Implications for Optimization}
\label{sec:optimization}

In this section, we provides the details of \autoref{thm:informal}, which justifies the importance of the dynamic regret guarantee of \alg{} in \ref{olu}, based on its implications for (non-convex) optimization.

Recall that in \ref{olu}, for the function $F$ we want to minimize, we choose the update $\update_t$ by the output of \alg{} on the loss sequence  $\bg_{1:t}$ that are stochastic gradients of $F$. Formally, we assume $F:\R^d \to \R$ is differentiable but not necessarily convex.
Following the notations of \citet{cutkosky2023optimal}, given an iterate $\bw$ and a random variable $z$, let $\sto(\bw,z)$ be the standard \emph{stochastic gradient oracle} of $F$ at $\bw$, satisfying $\E_z[\sto(\bw,z)]= \nabla F(\bw)$. 

We first introduce \citep[Theorem 7]{cutkosky2023optimal} that crucially connects the dynamic regret guarantee to the optimization guarantee, justifying the importance of the dynamic regret of \alg{}.

\begin{theorem}[{\bf Importance of dynamic regret in \ref{olu}}] \label{thm:guarantee} 
Consider the optimization algorithm \eqref{exp:update} where
\begin{itemize}[leftmargin=*,itemsep=0pt,topsep=0pt]
\item the update $\update_t$ is chosen by \alg{} based on $\bg_{1:t}$;
\item the gradient $\bg_t=\sto(\bw_{t-1}+s_t\update_t,z_t)$, where the i.i.d. samples $s_t\sim \textup{Unif}([0,1])$ and $z_t\sim z$.
\end{itemize}
Then, for all $T\geq 0$ and any comparator sequence $\bu_{0:T-1}$, the iterate $\bw_T$ generated by \eqref{exp:update} satisfies
\begin{equation}
\E\spar{F(\bw_{T})}- F(\bw_0) =  \E\spar{\sum_{t=1}^{T} \langle \bg_{t}, \bu_{t-1}\rangle +\regret_{T}(\bu_{0:T-1})},
\end{equation}
where $\regret_{T}(\bu_{0:T-1})$ is the regret bound of \alg{}.
\end{theorem}

The proof is provided in \autoref{pf:thm:guarantee}.
The insight is that the nonconvexity can be handled by randomization (through $s_t$) at the gradient query, and if $F$ is convex, it suffices to set $s_t=1$ (which is more aligned with practice). Regarding the quantitative result, \autoref{thm:guarantee} precisely captures the informal claim from \autoref{thm:informal}:
\begin{center}
\emph{Improving the dynamic regret of \alg{} directly leads to better optimization guarantee.}
\end{center}  
Hence,  the effectiveness of  \ref{dftrl} as a dynamic online learner (discussed in \autoref{sec:dynamic}) supports the success of its optimizer counterpart, namely the update \eqref{adam}.

To make this concrete, we revisit the lower bound examples from \autoref{thm:lower}, and discuss the implications.

\subsection{Revisiting lower bound example}
\label{sec:abstract}
Consider an abstract scenario where we fix the stochastic gradient sequence $\bg_{1:T}$ to be given the lower bound example of \autoref{thm:lower}.
We compare the guarantees in \autoref{thm:guarantee},  
\begin{align} \label{rhs}
\sum_{t=1}^{T} \langle \bg_{t}, \bu_{t-1}\rangle   +   \regret_{T}(\bu_{0:T-1})\,.
\end{align}  

Then, the following is a direct corollary of \autoref{sec:components}.
\begin{corollary} \label{cor:opt}
Suppose $\bg_{1:T}$ and $\bu_{0:T-1}$ are chosen as in the $2D$ example of  \autoref{thm:lower}. Then, the following statements hold:
\begin{itemize}[leftmargin=*,itemsep=0pt,topsep=0pt] 
\item \textbf{Adam.} If \alg{} is \ref{dftrl-clip} with $1-\beta = \Theta(T^{-\nicefrac{2}{3}})$ and $D=1$, then $\eqref{rhs} = -T +o(T)$.

\item \textbf{No momentum (\emph{e.g.}, SGD/AdaGrad).} If \alg{} has the form \eqref{no-momentum}, then $\eqref{rhs} \geq -3$.

   \item \textbf{No discounting.} If \alg{} is \ref{dftrl-clip} with $\beta=1$ and $D=1$, then $\eqref{rhs} \geq -\frac{1}{2}T -\frac{3}{2}$.

\end{itemize}
\end{corollary}

Essentially, this result specializes the key insight from \autoref{sec:components}, \emph{i.e.}, the benefits of the momentum and the discounting factor, to the corresponding optimizers.

On the other hand, we acknowledge that the abstract scenario of fixing $\bg_{1:T}$ to be some desired sequence is not entirely practical, because ($i$) they should be stochastic gradients of $F$ and ($ii$) $\bg_{1:T}$ depends on the \alg{} of choice.
Below, we build on the intuitions of this abstract example and present a concrete classification problem where Adam is more beneficial than the other two baselines.

\subsection{Adam Could Be Effective for Sparse and Nonstationary Gradients}
\label{sec:concrete}

Inspired by \citet{duchi10adagrad}, we propose a concrete classification scenario for which we see the performance gap described in \autoref{cor:opt}. 
At a high level, it is designed such that the associated gradient sequence imitates the one from our lower bound construction, \autoref{thm:lower}. 

\textbf{Classification of sparse data with small $\eta$.} Consider the classification of $(\bz_i,y_i)$ where $\bz_i \in \R^d$ is the data vector with its label $y_i \in \{\pm 1\}$. We assume that each data vector $\bz_i$ is \textbf{sparse}.
Concretely, we focus on the setting where the dataset consists of coordinate vectors and positive labels, \emph{i.e.},  $\{(\bz_i,y_i)\}_{i=1}^d$ where $\bz_i= \be_i$ and $y_i=1$ for $i=1,\dotsc,d$. 
\begin{figure}
\begin{center}
\scalebox{0.6}{ \begin{tikzpicture}
\begin{axis}[ 
grid=major, 
 xmin=-1, xmax=4,
ymin=0, ymax=2.1 , %
enlarge x limits=false,  
enlarge y limits=false,  
unit vector ratio=1 1 1,  
]
\addplot[domain=-2:1, blue, ultra thick] {max(1-x, 0) + 0.25*abs(x)};
\addplot[domain=1:4, blue, ultra thick] {0.25*abs(x)};
\end{axis}
\end{tikzpicture}}
\end{center}
    \caption{1D illustration of the regularized hinge loss 
    $\ell(x) = \max(0,1-x) + \lambda |x|$. We illustrate the case $\lambda=1/4$.}
\label{fig:hinge}
\end{figure}
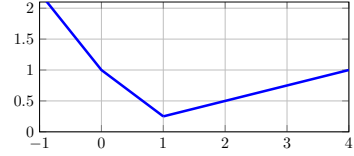
We consider the following regularized hinge loss 
\begin{align}
\ell(x) = \max(0,1 - x) + \lambda|x|\quad \text{for $\lambda <1$}\,, 
\end{align}
which prevents the classifier from becoming over-confident.
See \autoref{fig:hinge} for the landscape of this hinge loss.
Then, the training loss is given as
\begin{align}
F(\bw) =  \frac{1}{d} \sum_{i=1}^d  \ell (y_i\inp{\bz_i}{\bw}) \,,    
\end{align}
and during iteration $t$, assume the algorithm receives a single data $(\bz_{i(t)},y_{i(t)})$, where $i(t)$ is sampled from $\{1,\dotsc,d\}$ uniformly at random.

For experiments, we initialize at $\bw_0 = \mathbf{0}$ and use \textbf{a small learning rate},  $\eta = 0.01$, so that each coordinate takes multiple steps to approach the minimum $w=1$.
Besides, we choose $d=100$ and $\lambda=1/4$.  
Two different settings are considered:
\begin{enumerate}[leftmargin=*,itemsep=0pt,topsep=0pt]
\item {\bf Left plot of \autoref{fig:adam}:} $\bz_i= \be_i$ for $i=1,\dotsc, d$. 
\item {\bf Right plot of \autoref{fig:adam}:} $\bz_i= c_i\be_i$, where $c_i\sim \text{Unif}[0,2]$ for $i=1,\dotsc, d$. This setting allows the data vectors to have different magnitudes. 
\end{enumerate}
From \autoref{fig:adam}, SGD exhibits sluggish progress owing to the sparse nature of stochastic gradients—--that is, only one coordinate is updated at each step. Moreover, setting $\beta=1$ in Adam also results in suboptimal performance once the coordinate-wise iterate $\bw_t[i]$ exceeds $1$ (for some coordinate $i$)---after that, the stochastic gradients point toward other directions. In contrast, adopting Adam with $\beta<1$ effectively addresses these issues, adeptly managing \textbf{both the sparsity of updates and the non-stationarity of gradients}.

Next, we provide a possible qualitative explanation of this gap, from the dynamic regret perspective.

\begin{figure}
\begin{center} 
\includegraphics[width=0.5\textwidth]{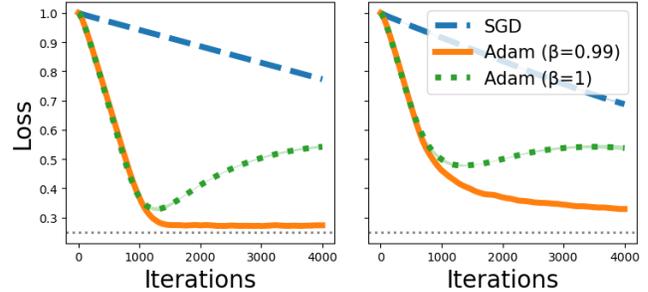}
\end{center}
\vspace{-10pt}
\caption{Experimental results for the hinge loss classification. ({\bf Left}) the case of  $\bz_i= \be_i$. ({\bf Right}) the case of  $\bz_i= c_i\be_i$ where $c_i\sim \text{Unif}[0,2]$. The horizontal dotted line indicates the optimum value of $F$. 
All experiments are run for five different random seeds, and we plot the error shades (they are quite small and not conspicuous).}
\label{fig:adam}
\end{figure}

\textbf{Qualitative dynamic regret analysis.}  
Since $i(t)$ is sampled uniformly, it suffices to focus on the first coordinate and consider the 1D setting for simplicity.
Then, since \textbf{learning rate $\eta$ is chosen small}, starting from  $w_0 = 0$, the above setting could be abstractly thought as generating the following sparse stochastic gradient sequence
\begin{align} \label{gt_abstract}
g_t = \begin{cases}
  (1-\lambda) \cdot \ind{i(t)=1} &\text{if }t\lesssim \tau\\
    -\lambda \cdot \ind{i(t)=1} &\text{if }t\gtrsim \tau
\end{cases}\,,
\end{align} 
where  $\tau$ denotes the first iteration such that $w_\tau>1$.
This can be seen as one of the simplest setting of a {\bf sparse} and {\bf non-stationary} gradient sequence, mirroring the construction from \autoref{thm:lower}.
Now we compare the 1D version of the guarantee \eqref{rhs}, \emph{i.e.}, the \emph{total loss} of the \alg{}
\begin{align} \label{rhs-1d}
\sum_{t=1}^{T}  g_{t} u_{t-1}   + 
\regret_{T}(u_{0:T-1})= \sum_{t=1}^{T}  g_{t} \Delta_{t-1}   
\end{align}  
for each \alg{} akin to \autoref{cor:opt}:
\begin{itemize}[leftmargin=*,itemsep=0pt,topsep=0pt]
\item {\bf No momentum.} 
Due to the sparsity in gradient sequence \eqref{gt_abstract}, having no momentum incurs a large dynamic regret.
More specifically, since $\E\abs{g_t g_{t-1}} \lesssim \frac{1}{d^2}$, we have
\begin{align}
    \E\abs{g_t  \Delta_{t-1}} \lesssim \frac{1}{d^2}\,, 
\end{align}
for ``non-momentum'' \alg{} of the form \eqref{no-momentum}.
Therefore, we have $\eqref{rhs-1d}=\sum_{t} g_t\Delta_{t-1}\gtrsim -  \frac{1}{d^2}T $. 
\item {\bf No discounting factor.} Due to the non-stationarity in gradient sequence \eqref{gt_abstract}, having no discounting factor also incurs a large dynamic regret.
In particular, \ref{dftrl} with $\beta=1$ generates the update $\Delta_t$ with the same sign as  $-\sum_t g_t$. In a typical run, the sign of $-\sum_t g_t$ remains unchanged throughout, but the sign of $g_t$ flips once $t \gtrsim \tau$.
 Hence, roughly speaking, the update $\Delta_t$ does not have the ``correct'' sign (corresponding to the \emph{descent} direction) after $t\gtrsim \tau$, leading to  $\eqref{rhs-1d}  \gtrsim - (1-\lambda)\tau$. 
\end{itemize}
In contrast, following \autoref{cor:opt}, Adam achieves $\eqref{rhs-1d}  \leq - (1-\lambda)\tau - \lambda(T-\tau) + o(T) = -\lambda T - (1-2\lambda)\tau  + o(T)$, improving the above when $\tau$ is small.

\section{Conclusion and Discussion}

This work presents a new perspective on the popular Adam optimizer, based on the framework of online learning of updates \eqref{olu} \citep{cutkosky2023optimal}.
Under \ref{olu}, our main observation is that Adam corresponds to choosing the dynamic version of FTRL that utilizes the discounting factor.
We find this perspective quite advantageous as it gives new insights into the role of Adam's algorithmic components, such as momentum and the exponential moving average. 

In fact, our perspective has already inspired a follow-up work by \citet{ahn2024adam}, where they show the optimal iteration complexity of Adam for finding stationary points under nonconvex and nonsmooth functions.
Their analysis crucially utilizes our perspective that Adam corresponds to \ref{dftrl} under \ref{olu}. 

In addition to \citep{ahn2024adam}, the findings in this work unlock several other important future directions. Below, we list a few of them.
\begin{itemize}[leftmargin=*,itemsep=0pt,topsep=0pt]
\item \textbf{Role of two discounting factors.} As an initial effort, this work considers the case of $\beta_1=\beta_2$. Given that the default choice in practice is $\beta_1 =0.9$ and $\beta_2 =0.999$, it would be important to understand the precise effect of choosing $\beta_1 <\beta_2$.

\item \textbf{Other algorithms based on \ref{olu}.} The framework \ref{olu} unlocks a new way to analyze optimization algorithms. As we highlighted in \autoref{sec:components}, \ref{olu} establishes the \emph{one-to-one correspondence} between other optimizers and their online learning counterparts.
The main scope of this work is to provide a better understanding of Adam specifically, and extending our framework to other popular optimizers, such as RMSProp~\citep{tieleman2012rmsprop}, AdaDelta~\citep{zeiler2012adadelta}, Lion~\citep{chen2023symbolic} etc, is an interesting future direction.
We believe that understand them based on our framework would offer new insights for them.

\item \textbf{Algorithm design based on \ref{olu}.} Our current dynamic regret analysis of \ref{dftrl} requires knowledge of the environment.
Developing a version of \ref{dftrl} that automatically adapts to the environment without prior knowledge might lead to more practical algorithms. 
Moreover, whether one can design practical optimizers based on recent advancements in dynamic online learning (\emph{e.g.} \citet{jacobsen2022parameter,zhang2023unconstrained}) would be an important future direction.

\item {\bf Fine-grained analysis of Adam for practical settings.}  As discussed earlier, Adam has gained significant attention due to its effectiveness in training language models.   Recently, \citet{kunstner2024heavy} investigate key characteristics of the language modeling datasets that might have caused the difficulties in training. In particular, they identify the \emph{heavy-tailed imbalance} property, where  there are a lot more infrequent words/tokens than frequent ones in most language modeling datasets. Further, they demonstrate this property as a main reason why Adam is particularly effective at language modeling tasks \cite{zhang2020adaptive}. 
We find their main insights consistent with our claim in \autoref{sec:concrete}. 
The infrequent words in the dataset would likely lead to sparse and non-stationary gradients.
Formally investigating this would be also an interesting future direction.

\end{itemize}

\section*{Acknowledgements}
Kwangjun Ahn is indebted to Ashok Cutkosky for several inspiring conversations that led to this project.

Kwangjun Ahn was supported by the ONR grant (N00014-23-1-2299), MIT-IBM Watson, a Vannevar Bush fellowship from Office of the Secretary of Defense, and NSF CAREER award (1846088). Yunbum Kook was supported in part by NSF awards CCF-2007443 and CCF-2134105. Zhiyu Zhang was supported by the funding from Heng Yang. 

\section*{Impact Statement}

This paper provides a new perspective of understanding the Adam optimizer.  
This work is theoretical, and we do not see any immediate potential societal consequences.

\bibliography{ref}
\bibliographystyle{plainnat}
\onecolumn
\appendix

\section{Analysis of scale-free FTRL}\label{pf:thm:dftrl_discount_regret}

Recall from \autoref{sec:olu} that our construction starts with a gradient adaptive FTRL algorithm called \emph{scale-free FTRL} \citep{orabona2018scale}. This section presents a self-contained proof of its undiscounted static regret bound. 

Formally, we consider the 1D OLO problem introduced at the beginning of \autoref{sec:olu}. Scale-free FTRL is defined as \ref{ftrl} with the step size $\eta_t  = \nicefrac{\alpha}{\sqrt{\sum_{s=1}^t v_s^2}}$, where $\alpha>0$ is a scaling factor. Equivalently, it has the update rule
\begin{align}
\Delta_t =  \argmin_{x}\left[  
 \frac{1}{\eta_t}|x|^2+ \sum_{s=1}^{t} v_{s} x  \right]  = -\eta_t \sum_{s=1}^{t} v_{s} = -\alpha \frac{\sum_{s=1}^{t} v_{s}}{\sqrt{\sum_{s=1}^t v_s^2}}\,.   
\end{align}
For well-posedness, if the denominator $\sqrt{\sum_{s=1}^t v_s^2}=0$, then we set the update to be $\Delta_t=0$.

\begin{theorem}[Static regret of scale-free FTRL] \label{thm:adagrad-ftrl}
For all $T>0$, loss sequence $v_{1:T}$ and comparator $u\in\R$, scale-free FTRL guarantees the following static regret bound
\begin{equation*}
\sum_{t=1}^{T}v_{t}(\Delta_{t-1}-u)\leq \rpar{\frac{u^2}{2\alpha}+\sqrt{2}\alpha}\sqrt{\sum_{t=1}^{T}v_t^2}+2 \rpar{\max_{t\in[1,T]}\abs{\Delta_t}}\rpar{\max_{t\in[1,T]}\abs{v_t}}\,.
\end{equation*}
\end{theorem}

We remark that \citet{erven2021why} directly applies the clipping technique from \citep{cutkosky2019artificial} to obtain a similar regret bound as \autoref{thm:adagrad-ftrl}, but in this way the associated algorithm is \emph{scale-free FTRL on the ``clipped'' gradients}, rather than scale-FTRL itself.
In contrast, we analyze the original scale-free FTRL algorithm for the purpose of explaining Adam (since in practice, Adam does not use the \emph{Cutkosky-style clipping} on the stochastic gradients). This requires a slightly more involved analysis. 

\paragraph{Comparison with \citep{orabona2018scale}.} Before proving this theorem, we compare our regret bound with that of scale-free FTRL from \citep[Theorem~1]{orabona2018scale}. Their regret bound in the unconstrained domain setting (which means the domain diameter $D$ defined in their Theorem 1 is infinite) is 
\begin{equation*}
\sum_{t=1}^{T}v_{t}(\Delta_{t-1}-u)\leq \O\rpar{\rpar{u^2+1}\sqrt{\sum_{t=1}^{T}v_t^2}+\sqrt{T} \max_{t\in[1,T]}\abs{v_t}}\,.
\end{equation*}
Our bound replaces the $\sqrt{T}$-factor by the maximum output magnitude (\emph{i.e.}, $\max_{t\in[1,T]}\abs{\Delta_t}$), and our is better since
\begin{align*}
\abs{\Delta_t}=\alpha \frac{\abs{\sum_{s=1}^{t} v_{s}}}{\sqrt{\sum_{s=1}^t v_s^2}}\leq \alpha\sqrt{t}\,,
\end{align*}
which follows from the Cauchy-Schwarz inequality. We need such an improvement because in the discounted setting, the scaled loss sequence will have rapidly growing magnitude, which means this Cauchy-Schwarz step would be quite loose.

Our proof makes a nontrivial use of the gradient clipping technique from \citep{cutkosky2019artificial}, which is also different from \citep[Theorem~1]{orabona2018scale} and could be of independent interest. However, we acknowledge that directly modifying the argument of \cite{orabona2018scale} might achieve a similar goal.

\subsection{Proof of \autoref{thm:adagrad-ftrl}}
On the high level, the proof carefully combines the standard FTRL analysis, \emph{e.g.}, \citep[Lemma~7.1]{orabona2019modern}, and the gradient clipping technique of \cite{cutkosky2019artificial}.

\paragraph{Step 1.} We start with a preparatory step. Let $\tau$ be the index such that $v_t=0$ for all $t\leq\tau$, and $v_{\tau+1}\neq 0$. Without loss of generality, assume $T>\tau$. Then,
\begin{align*}
\sum_{t=1}^{T}v_{t}(\Delta_{t-1}-u)=\sum_{t=\tau+1}^{T}v_{t}(\Delta_{t-1}-u)\,.
\end{align*}
On the RHS we have $\Delta_\tau=0$, and for all $t>\tau$, $\Delta_t$ is now well-defined by the ``nice'' gradient adaptive update rule (\emph{i.e.}, the denominator does not cause a problem)
\begin{align}
\Delta_t= -\eta_t \sum_{s=1}^{t} v_{s} =  -\alpha \frac{\sum_{s=1}^{t} v_{s}}{\sqrt{\sum_{s=1}^t v_s^2}}\,.   
\end{align}
To proceed, for all $t>\tau$, we define $F_t(x)=\frac{1}{2\eta_t}|x|^2+\sum_{s=1}^{t}v_sx$, which means that $\Delta_t=\argmin_x F_t(x)$. Trivially, at the time index $\tau$, we define $F_\tau(x)=0$ for all $x$.

\paragraph{Step 2.} The main part of the proof starts from the standard FTRL equality \citep[Lemma~7.1]{orabona2019modern},
\begin{align*}
\sum_{t=1}^{T}v_{t}(\Delta_{t-1}-u)&= \sum_{t=1}^{T}v_{t}\Delta_{t-1}- \left(F_{T}(u) -\frac{1}{\eta_T} |u|^2\right)\\
&=\frac{1}{\eta_T} |u|^2 +\sum_{t=\tau}^{T-1}\spar{F_t(\Delta_t)-F_{t+1}(\Delta_{t+1})+v_{t+1}\Delta_t}+F_{T}(\Delta_T)-F_{T}(u)\\
&\leq \frac{1}{\eta_T} |u|^2 +\sum_{t=\tau}^{T-1}\spar{F_t(\Delta_t)-F_{t+1}(\Delta_{t+1})+v_{t+1}\Delta_t}\,,
\end{align*}
where the last inequality follows since $\Delta_T=\argmin_x F_T(x)$.

Consider the terms $F_t(\Delta_t)-F_{t+1}(\Delta_{t+1})+v_{t+1}\Delta_t$ in the above sum. Let us define the clipped gradient 
\begin{align}
\tilde v_t \coloneqq \clip_{\sqrt{\sum_{s=1}^{t-1}v_s^2}}(v_t)\,,
\end{align}
where for any $D\geq 0$, $\clip_D(x):=x \min(\frac{D}{|x|},1)$. 
\begin{itemize}
\item For all $t>\tau$, we have
\begin{align*}
&\quad~ F_t(\Delta_t)-F_{t+1}(\Delta_{t+1})+v_{t+1}\Delta_t\\
&=F_t(\Delta_t)+v_{t+1}\Delta_t-F_t(\Delta_{t+1})-v_{t+1}\Delta_{t+1}+\frac{1}{2\eta_t}\abs{\Delta_{t+1}}^2-\frac{1}{2\eta_{t+1}}\abs{\Delta_{t+1}}^2\\
&\leq F_t(\Delta_t)+v_{t+1}\Delta_t-F_t(\Delta_{t+1})-v_{t+1}\Delta_{t+1}\\
&\leq F_t(\Delta_t)+\tilde v_{t+1}\Delta_t-F_t(\Delta_{t+1})-\tilde v_{t+1}\Delta_{t+1}+\abs{v_{t+1}-\tilde v_{t+1}}\rpar{\abs{\Delta_t}+\abs{\Delta_{t+1}}}\,.
\end{align*}
Following a standard fact of convex functions, e.g., \citep[Lemma~7.8]{orabona2019modern}, since $F_t(\Delta)+\tilde v_{t+1}\Delta$ is $\frac{1}{\eta_t}$-strongly convex, it holds that
\begin{equation*}
F_t(\Delta_t)+\tilde v_{t+1} \Delta_t-F_t(\Delta_{t+1})-\tilde v_{t+1}\Delta_{t+1}\leq F_t(\Delta_t)+\tilde v_{t+1}\Delta_t-\min_{\Delta}\spar{F_t(\Delta)+\tilde v_{t+1}\Delta}\leq \frac{\eta_{t}}{2}\tilde v_{t+1}^2\,.
\end{equation*}
\item As for the case of $t=\tau$, since $F_\tau(\Delta_\tau)=0$ and $\Delta_\tau=0$,
\begin{align*}
F_\tau(\Delta_\tau)-F_{\tau+1}(\Delta_{\tau+1})+v_{\tau+1}\Delta_\tau&=-F_{\tau+1}(\Delta_{\tau+1})
=-\frac{1}{\eta_{\tau+1}}|\Delta_{\tau+1}|^2-\sum_{s=1}^{\tau+1}v_s\Delta_{\tau+1}
\leq -v_{\tau+1}\Delta_{\tau+1}\\
&\leq \abs{v_{\tau+1}-\tilde v_{\tau+1}}\abs{\Delta_{\tau+1}}\,.\tag{$\tilde v_{\tau+1}=0$}
\end{align*}
\end{itemize}
Thus, we obtain the following bound:
\[
\sum_{t=1}^{T}v_{t}(\Delta_{t-1}-u)\leq \frac{u^2}{2\alpha}\sqrt{\sum_{t=1}^{T}v_t^2}+\frac{\alpha}{2}\sum_{t=\tau+1}^{T-1}\frac{\tilde v_{t+1}^2}{\sqrt{\sum_{s=1}^{t}v_s^2}}+\sum_{t=1}^{T}\abs{v_{t}-\tilde v_{t}}\rpar{\abs{\Delta_{t-1}}+\abs{\Delta_{t}}}\,.
\]

\paragraph{Step 3.} Finally, consider the two summation terms on the RHS one-by-one. We begin with the first term. 
\begin{align}
\frac{\tilde v_{t+1}^2}{\sqrt{\sum_{s=1}^{t}v_s^2}}&=  \sqrt{2}\frac{\tilde v_{t+1}^2}{\sqrt{2\sum_{s=1}^{t}v_s^2}}\leq \sqrt{2}\frac{\tilde v_{t+1}^2}{\sqrt{\tilde v_{t+1}^2+\sum_{s=1}^{t} v_s^2}}=2\sqrt{2}\frac{\tilde v_{t+1}^2}{2\sqrt{\tilde v_{t+1}^2+\sum_{s=1}^{t} v_s^2}} \\
&\leq 2\sqrt{2}\frac{\tilde v_{t+1}^2}{\sqrt{\tilde v_{t+1}^2+\sum_{s=1}^{t} v_s^2} + \sqrt{ \sum_{s=1}^{t} v_s^2}} = 2\sqrt{2}\rpar{\sqrt{\tilde v_{t+1}^2+\sum_{s=1}^{t}v_s^2}-\sqrt{\sum_{s=1}^{t}v_s^2}}\,.
\end{align}
Thus, it follows that
\begin{align*}
\sum_{t=\tau+1}^{T-1}\frac{\tilde v_{t+1}^2}{\sqrt{\sum_{s=1}^{t}v_s^2}}&\leq 2\sqrt{2}\sum_{t=\tau+1}^{T-1}\rpar{\sqrt{\tilde v_{t+1}^2+\sum_{s=1}^{t}v_s^2}-\sqrt{\sum_{s=1}^{t}v_s^2}}
\leq 2\sqrt{2}\sum_{t=\tau+1}^{T-1}\rpar{\sqrt{ \sum_{s=1}^{t+1}v_s^2}-\sqrt{\sum_{s=1}^{t}v_s^2}}\\
&\leq 2\sqrt{2}\sqrt{\sum_{t=1}^{T}v_t^2}\,.
\end{align*}
As for the second summation, we handle it similarly to \citep[Theorem~2]{cutkosky2019artificial}.
Defining $G_t =\max_{s\in[1,t]}\abs{v_s}$, since $|\Delta_0|=0$, we have
\begin{align*}
\sum_{t=1}^{T}\abs{v_{t}-\tilde v_{t}}\rpar{\abs{\Delta_{t-1}}+\abs{\Delta_{t}}}&\leq 2\rpar{\max_{t\in[1,T]}\abs{\Delta_t}}\sum_{t=1}^T\abs{v_t-\tilde v_t}
= 2\max\rpar{0,\max_{t\in[1,T]}\abs{\Delta_t}}\sum_{t=1}^T\rpar{\abs{v_t}-\sqrt{\sum_{s=1}^{t-1}v_s^2}}\\
&\leq2\max\rpar{0,\max_{t\in[1,T]}\abs{\Delta_t}}\sum_{t=1}^T\rpar{G_t-G_{t-1}} \\
&\leq 2\rpar{\max_{t\in[1,T]}\abs{\Delta_t}}G_T\,.
\end{align*}    
Combining the two upper bounds above completes the proof.

\section{Analysis of \ref{dftrl}}
\label{pf:dftrl}
As discussed in \autoref{sec:olu},  \ref{adam} corresponds to \ref{dftrl}, the discounted version of scale-free FTRL, through the \ref{olu} framework. Thus, quantifying \ref{adam}'s performance comes down to analyzing the dynamic regret of \ref{dftrl}.

We now present the complete version of \autoref{thm:dftrl} (the dynamic regret bound of \ref{dftrl}), fleshing out the proof sketch in \autoref{sec:sketch}. 
A main proof ingredient is our \emph{discounted-to-dynamic conversion}. 
As a quick reminder, the formal setting considered is still the 1D OLO problem, where the output of the algorithm is denoted by $\Delta_t\in\R$, and the loss function is denoted by $\ell_t(x) = v_{t+1}x$ with $v_{t+1}\in \R$.

\paragraph{Step 1: Discounted regret.} Our analysis starts with a concept called \emph{discounted regret}, formalized in \autoref{def:discounted_regret}.
We recall the definition below for reader's convenience.
\begin{definition} [{$\beta$-discounted regret}] 
For any discounting factor $\beta\in(0,1]$, the $\beta$-discounted regret is defined as
\[
\regret_{T;\beta}(u)\coloneqq \sum_{t=1}^{T} \beta^{T-t}v_{t}(\Delta_{t-1} -u)\,.
\]
\end{definition}
\noindent 
When $\beta=1$, the $\beta$-discounted regret recovers the standard static regret $\regret_{T}(u)$. The notational difference is simply an extra subscript $\beta$ in the $\beta$-discounted regret, \emph{i.e.}, $\regret_{\cdot;\beta}$. 

Intuitively, \ref{dftrl} should achieve good $\beta$-discounted regret, as long as scale-free FTRL achieves good static regret.
This intuition follows from observations that \ref{dftrl} is just scale-free FTRL with the ``discounted losses'' $v_t\leftarrow\beta^{-t}v_t$, and that the $\beta$-discounted regret considers the loss sequence $\beta^{-t} v_{t}$ instead of $v_t$. We formalize this with the proof in \autoref{proof:dftrl_discount_regret}.

\begin{theorem}[Discounted regret of \ref{dftrl}] \label{thm:dftrl_discount_regret}
For all $T>0$, loss sequence $v_{1:T}$ and comparator $u\in\R$, \ref{dftrl} guarantees the $\beta$-discounted regret bound
\[
\regret_{T;\beta}(u)\leq \left (\frac{u^2}{2\alpha}+\sqrt 2\alpha\right )\sqrt{\sum_{t=1}^{T} (\beta^{T-t}v_t)^2}+2\left (\max_{t\in [1,T]} \lvert \Delta_t\rvert\right )\left (\max_{t\in [1,T]} \lvert \beta^{T-t} v_t\rvert\right )\,.
\]
\end{theorem}

\paragraph{Step 2: Discounted-to-dynamic conversion.} The remaining task is to convert this discounted regret bound to a dynamic regret bound. We accomplish this via a general discounted-to-dynamic conversion, which is of independent interest. The idea is to partition the entire time horizon into subintervals, and then consider static comparators on each of them. 
To be precise, we consider a \emph{partition} of $[1,T]$ denoted by $\bigcup_{i=1}^N[a_i,b_i]$, where $a_1=1$, $b_i+1= a_{i+1}$ for all $1\le i<N-1$, and $b_N=T$.
Then each partitioned interval $[a_i,b_i]$ is coupled with an arbitrary fixed comparator $\bar u_i$. 

We remark that this conversion is independent of the algorithm. For an algorithm $\A$, $\regret^\A_{T;\beta}(u)$ and $\regret^\A_{T}(u_{0:T-1})$ denote its $\beta$-discounted regret (\autoref{def:discounted_regret}) and its dynamic regret \eqref{exp:dynamic_regret}, respectively. See \autoref{pf:thm:discount_to_dynamic} for the proof.

\begin{theorem}[\textbf{Discounted-to-dynamic conversion}] \label{thm:discount_to_dynamic}
Consider an arbitrary 1D OLO algorithm $\A$. For all $T>0$, loss sequence $v_{1:T}$ and comparator sequence $u_{0:T-1}$, the dynamic regret of $\A$ satisfies
\begin{align*}
\regret^\A_{T}(u_{0:T-1})&= 
\beta \regret^\A_{T;\beta}(\bar u_N) +(1-\beta)\sum_{i=1}^N\sum_{t\in [a_i,b_i]}  \regret^\A_{t;\beta}(\bar u_i)\\
&\quad +\beta \sum_{i=1}^{N-1} \spar{\rpar{\sum_{t=1}^{b_i} \beta^{b_i-t} v_{t}}(\bar u_{i+1} - \bar u_{i})} +\sum_{i=1}^N\sum_{t\in [a_i,b_i]} v_{t}(\bar{u}_i-u_{t-1})\,,
\end{align*}
where $\bigcup_{i=1}^N[a_i,b_i]$ is an arbitrary partition of of $[1,T]$, and $\bar u_1,\ldots \bar u_N\in\R$ are also arbitrary.
\end{theorem}  

The remarkable aspect of this result is that it is an \emph{equality}. That is, we do not lose anything through the conversion. 
Given a discounted regret bound of $\A$, we can make use of this conversion by substituting $\regret^\A_{T;\beta}(\bar u_N)$ and $\regret^\A_{t;\beta}(\bar u_i)$ with their discounted regret bounds, and then taking the infimum on the RHS \emph{w.r.t.} the partition $\cup_{i \in [N]} [a_i,b_i]$ and the choice of the ``approximated comparator sequence'' $\bar u_1,\dotsc, \bar u_N$.

\paragraph{Step 3: Plugging in \ref{dftrl}.} Now we set $\A$ in the conversion to \ref{dftrl}. See \autoref{pf:dftrl_formal} for the proof.

\begin{thmbox}[{\bf Dynamic regret of \ref{dftrl}}] \label{thm:dftrl_formal}
Consider \ref{dftrl} with a fixed $\alpha>0$.
Consider any loss sequence $v_{1:T}$  and any comparator sequence $u_{0:T-1}$ s.t. $|u_t|\leq U$.
The dynamic regret \eqref{exp:dynamic_regret} of the \ref{dftrl} is bounded as
\begin{align}
\regret_{T}(u_{0:T-1}) &\leq \rpar{\frac{U^2}{2\alpha}+\sqrt{2}\alpha}\spar{\beta   \sqrt{V_\beta(v_{1:T})}   + (1-\beta) \sum_{t=1}^T  \sqrt{V_{\beta}(v_{1:t})}}\\
&\quad +2\beta \left (\max_{t\in [1,T]} \lvert \Delta_t \rvert \cdot \max_{t\in [1,T]} \lvert  \beta^{T-t} v_t\rvert\right ) +2(1-\beta) \sum_{t=1}^T  \left (\max_{s\in [1,t]} \lvert \Delta_s\rvert \cdot \max_{s\in [1,t]} \lvert   \beta^{t-s}v_s\rvert\right )\\
&\quad + \variation\,,
\end{align}
where  $V_\beta(v_{1:t})\coloneqq \sum_{s=1}^{t} (\beta^{t-s}v_s)^2$ is the \emph{discounted variance of the losses} and 
\begin{align}
\variation \coloneqq\inf \left\{ \beta \sum_{i=1}^{N-1}
\rpar{\sum_{t=1}^{b_i} \beta^{b_i-t} v_{t}} \rpar{\bar u_{i+1} - \bar u_{i}} +  \sum_{i=1}^N\sum_{t\in [a_i,b_i]} v_{t}(\bar{u}_i-u_{t-1})\right\}\,,
\end{align}
and the infimum in $\variation$ is taken over
all partitions $\bigcup_{i=1}^N [a_i,b_i]$ of $[1,T]$ and all choices of $\{\bar u_i\}_{i\in N}$ satisfying $|\bar u_i|\leq U$.
\end{thmbox} 

In \autoref{thm:dftrl_formal}, the variation term $\variation$ consists of two terms.
The first part
\begin{align}
\beta \sum_{i=1}^{N-1}
\rpar{\sum_{t=1}^{b_i} \beta^{b_i-t} v_{t}} \rpar{\bar u_{i+1} - \bar u_{i}}
\end{align}
measures how fast the representative comparators $\bar u_i$'s change across different subintervals, and we hence call it the ``inter-partition variation". The second term
\[
\sum_{i=1}^N\sum_{t\in [a_i,b_i]} v_{t}(\bar{u}_i-u_{t-1})
\]
measures how different $u_t$'s are from the representative comparators $\bar u_i$'s within each subinterval, and we call it the  ``intra-partition variation". A notable strength of the variation term is that it is the infimum over all partitions and $\bar u_i$'s. In other words, the upper bound will \textbf{automatically adapt} to the best choice of partitions and $\bar u_i$'s without knowing them explicitly. For instance, choosing
\begin{align}
\bar u_i = \frac{\sum_{t\in[a_i,b_i]}v_{t}u_{t-1}}{\sum_{t\in[a_i,b_i]}v_{t}}  
\end{align}
would make the intra-partition variation term zero.

Referring to \autoref{thm:dftrl_formal}, we immediately obtain its simplified version in the main text, \autoref{thm:dftrl} with the proof in \autoref{sec:simplify}.
For the case of bounded comparators (\ref{dftrl-clip}), the dynamic regret can be analyzed using almost the same strategy, which leads us to state \autoref{thm:dftrl-clip} with the proof in \autoref{sec:dftrl-clip}.

\subsection{Proof of \autoref{thm:dftrl_discount_regret}}\label{proof:dftrl_discount_regret}

\ref{dftrl} is equivalent to scale-free FTRL with $v_t\leftarrow\beta^{-t}v_t$. Therefore, applying \autoref{thm:adagrad-ftrl} with $v_t \leftarrow \beta^{-t}v_t$ leads to
\[
\sum_{t=1}^{T}\beta^{-t}v_{t}(\Delta_{t-1}-u)\leq \rpar{\frac{u^2}{2\alpha}+\sqrt{2}\alpha}\sqrt{\sum_{t=1}^{T}\rpar{\beta^{-t}v_t}^2}+2 \rpar{\max_{t\in[1,T]}\abs{\Delta_t}}\rpar{\max_{t\in[1,T]}\abs{\beta^{-t}v_t}}\,.
\]
Multiplying both sides by $\beta^T$ completes the proof.

\subsection{Proof of \autoref{thm:discount_to_dynamic}}
\label{pf:thm:discount_to_dynamic}

Overall, the proof draws inspiration from \citep{zhang2018dynamic}, where a similar partitioning argument was used to prove a dynamic regret guarantee of a \emph{strongly adaptive} online learner \citep{daniely2015strongly}. Throughout the proof, we will omit the superscript $\A$ for brevity, since our argument is independent of specific algorithms. 

We start with a simple fact that connects the dynamic regret to the subinterval static regret.
For any partition $\bigcup_{i=1}^N[a_i,b_i]$ of $[1,T]$ and any choices of $\{\bar u_i\}_{i\in N}$, 
\begin{align}\label{ineq:1}
\regret_{T}(u_{0:T-1})  = \sum_{i=1}^N \sum_{t=a_i}^{b_i}v_{t}(\Delta_{t-1}-\bar u_i) +  \sum_{i=1}^N\sum_{t=a_i}^{b_i} v_{t}(\bar{u}_i-u_{t-1})\,.
\end{align}

To handle the static regret on the RHS,  we use the following result.
\begin{lemma}  \label{lem:to_static}
On any subinterval $[a,b]\subset [1,T]$, with any $u\in\R$,
\begin{align}
\sum_{t=a}^{b}v_{t}(\Delta_{t-1}-u)  = (1-\beta)\sum_{t=a}^b \regret_{t;\beta}(u) + \beta \rpar{\regret_{b;\beta}(u) -  \regret_{a-1;\beta}(u)}\,.
\end{align}
\end{lemma}
\begin{proof}
For all $t$, notice that
\begin{align}
\regret_{t;\beta}(u)= \sum_{s=1}^{t} \beta^{t-s}v_{s}(\Delta_{s-1} -u) \quad \text{and} \quad
\regret_{t-1;\beta}(u)=\sum_{s=1}^{t-1} \beta^{t-s}v_{s}(\Delta_{s-1} -u)\,,
\end{align}
and thus
\begin{align*}
\regret_{t;\beta}(u) - \beta\regret_{t-1;\beta}(u)  = v_{t}(\Delta_{t-1}-u)\,.
\end{align*}
Summing over $t\in[a,b]$,
\begin{align*}
\sum_{t=a}^{b}v_{t}(\Delta_{t-1}-u)&= \sum_{t=a}^b\regret_{t;\beta}(u)-\beta\sum_{t=a}^b\regret_{t-1;\beta}(u)\\
&=(1-\beta)\sum_{t=a}^b\regret_{t;\beta}(u)-\beta\regret_{a-1;\beta}(u)+\beta\regret_{b;\beta}(u)\,.\qedhere
\end{align*}
\end{proof}

Next, applying \autoref{lem:to_static} to each $[a_i,b_i]$ in \eqref{ineq:1} yields:
\begin{align}
\regret_{T}(u_{0:T-1}) = (1-\beta)\sum_{i=1}^N\sum_{t\in[a_i,b_i]} \regret_{t;\beta}(\bar u_i)+\beta \sum_{i=1}^N \spar{\regret_{b_i;\beta}(\bar{u}_i) -  \regret_{a_i-1;\beta}(\bar{u}_i)} 
+ \sum_{i=1}^N\sum_{t\in[a_i,b_i]} v_{t}(\bar{u}_i-u_{t-1})\,.
\end{align}
Since $a_{i}-1 = b_{i-1}$, the second term on the RHS can be rewritten as
\begin{align}
\sum_{i=1}^N \spar{\regret_{b_i;\beta}(\bar{u}_i) -  \regret_{a_i-1;\beta}(\bar{u}_i)} & =  \regret_{T;\beta}(\bar u_N) + \sum_{i=1}^{N-1} \spar{\regret_{b_i;\beta}(\bar u_{i}) -  \regret_{b_i;\beta}(\bar u_{i+1})}\\
&=  \regret_{T;\beta}(\bar u_N) + \sum_{i=1}^{N-1} \spar{\rpar{\sum_{t=1}^{b_i} \beta^{b_i-t} v_{t}}(\bar u_{i+1} - \bar u_{i})} \,.
\end{align} 
Combining everything above completes the proof.

\subsection{Proof of \autoref{thm:dftrl_formal}}\label{pf:dftrl_formal}

Due to \autoref{thm:discount_to_dynamic}, for any ``approximated comparator sequence'' $\bar u_1,\ldots,\bar u_N\in\R$ with $\abs{\bar u_i} \leq U$ for $i\in [N]$, we have 
\begin{align*}
\regret^\A_{T}(u_{0:T-1})&= 
\beta \regret^\A_{T;\beta}(\bar u_N) +(1-\beta)\sum_{i=1}^N\sum_{t\in [a_i,b_i]}  \regret^\A_{t;\beta}(\bar u_i) \\
&\quad +\beta \sum_{i=1}^{N-1} \spar{\rpar{\sum_{t=1}^{b_i} \beta^{b_i-t} v_{t}}(\bar u_{i+1} - \bar u_{i})} +\sum_{i=1}^N\sum_{t\in [a_i,b_i]} v_{t}(\bar{u}_i-u_{t-1})\,.
\end{align*}

Using \autoref{thm:dftrl_discount_regret} and $\abs{\bar u_i}\leq U$, 
\begin{align*}
\regret^\A_{T;\beta}(\bar u_N)&\leq \left (\frac{\bar u_N^2}{2\alpha}+\sqrt 2\alpha\right )\sqrt{\sum_{t=1}^{T} (\beta^{T-t}v_t)^2}+2\left (\max_{t\in [1,T]} \lvert \Delta_t\rvert\right )\left (\max_{t\in [1:T]} \lvert \beta^{T-t} v_t\rvert\right )\\
&\leq \rpar{\frac{U^2}{2\alpha}+\sqrt{2}\alpha}\sqrt{V_\beta(v_{1:T})}+2 \max_{t\in [1,T]} \lvert \Delta_t\rvert \cdot \max_{t\in [1:T]} \lvert  \beta^{T-t} v_t\rvert \,,
\end{align*}
and similarly, for any $t$ and $\bar u_i$,
\[
\regret^\A_{t;\beta}(\bar u_i)\leq \rpar{\frac{U^2}{2\alpha}+\sqrt{2}\alpha}\sqrt{V_\beta(v_{1:t})}+2 \max_{s\in [1,t]} \lvert \Delta_s\rvert \cdot \max_{s\in [1:t]} \lvert  \beta^{t-s} v_s\rvert \,.
\]
Putting these bounds into the equality and taking the infimum on the RHS (over the partition and the $\{\bar u_i\}_{i\in N}$ sequence satisfying $\abs{\bar u_i}\leq U$) complete the proof.

\subsection{Simplification for unbounded domain: \autoref{thm:dftrl}}\label{sec:simplify}

\autoref{thm:dftrl} follows as a corollary of  \autoref{thm:dftrl_formal} with  the partition $\bigcup_{t=1}^T\{t\}$,  $U=\alpha D_\beta$ and $\bar u_t = u_t$ for all $t$. 
With such choices, we have
\[
\variation  \leq  \beta \sum_{t=1}^{T-1}
\rpar{\sum_{s=1}^{t} \beta^{t-s} v_{s}} \rpar{  u_{t} -   u_{t-1}}\,.
\]
Recall that  $\cc \coloneqq \max_{t\in[1,T]} \frac{ \abs{\sum_{s=1}^t\beta^{t-s}v_s}}{ \sqrt{   \sum_{s=1}^t (\beta^{t-s}v_s)^2}}$. 
Hence, it follows that 
\[
\abs{\sum_{s=1}^{t} \beta^{t-s} v_{s}}\leq \cc \sqrt{V_\beta(v_{1:t})}\,.
\]
Therefore, the $\variation$ term in \autoref{thm:dftrl_formal} is reduced to
\[
\variation \leq  \beta \cc \sum_{t=1}^{T-1} \sqrt{V_\beta(v_{1:t})}\abs{u_{t} - u_{t+1}}\,,
\]
and thus the bound becomes (notice that $U=\alpha \cc$ and $\beta <  1$)
\begin{align}
\regret_{T}(u_{0:T-1}) &\leq \rpar{\frac{1}{2}\alpha \cc^2 +\sqrt{2}\alpha}\spar{\sqrt{V_\beta(v_{1:T})}   + (1-\beta) \sum_{t=1}^T  \sqrt{V_{\beta}(v_{1:t})}}+2 \alpha \cc G\spar{1+(1-\beta)T}\\
&\quad + \cc \sum_{t=1}^{T-1} \sqrt{V_\beta(v_{1:t})}\abs{u_{t} - u_{t-1}}\,.\label{eq:reduction_before_beta}
\end{align}
For all $t$, using $\beta < 1$
\[
V_\beta(v_{1:t})=\sum_{s=1}^{t} (\beta^{t-s}v_s)^2\leq G^2\sum_{i=0}^\infty \beta^{2i}=\frac{G^2}{1-\beta^2}<\frac{G^2}{1-\beta}\,.
\]
Therefore, 
\begin{align}
\regret_{T}(u_{0:T-1}) &\leq \rpar{\frac{1}{2}\alpha \cc^2 +\sqrt{2}\alpha}\spar{\frac{G}{\sqrt{1-\beta}}   + \sqrt{1-\beta}GT}+2\alpha \cc G\spar{1+(1-\beta)T}
+ \frac{\cc GP}{ \sqrt{1-\beta}}\\
& =\O\rpar{\rpar{\alpha+ \alpha \cc^2 +\cc P }\frac{G}{\sqrt{1-\beta}}+\rpar{\alpha+\alpha \cc^2}\sqrt{1-\beta}GT}\,.
\end{align}
Hence, we arrive at \autoref{thm:dftrl} presented in the main paper.

\subsection{Simplification for bounded domain \autoref{thm:dftrl-clip}}
\label{sec:dftrl-clip}

For the case of bounded comparators, \emph{i.e.}, $\abs{u_t}\leq D$, we consider \ref{dftrl-clip}, the $D$-clipped version of \ref{dftrl}:
\begin{align} \tag{$\beta$-FTRL$_{D}$} \label{dftrl-clip-restate}
\Delta_t = -\clip_{D}\rpar{\alpha \frac{ \sum_{s=1}^t\beta^{t-s}v_s}{ \sqrt{   \sum_{s=1}^t (\beta^{t-s}v_s)^2}}}\,,
\end{align}
where $\clip_D(x)=x\min(\frac{D}{|x |},1)$.
With \ref{dftrl-clip}, since $|\Delta_t|\leq D$ at each step, the following regret bound holds. The proof boils down to verifying that the entire proof strategy of \autoref{thm:dftrl_discount_regret} goes through even with projection. 

\begin{theorem}[Discounted regret of \ref{dftrl-clip}] \label{thm:dftrl-clip_discount_regret}
For all $T>0$, loss sequence $v_{1:T}$ and comparator $|u|\leq D$, the discounted regret bound of \ref{dftrl-clip} is
\[
\regret_{T;\beta}(u)\leq \left (\frac{u^2}{2\alpha}+\sqrt 2\alpha\right )\sqrt{\sum_{t=1}^{T} (\beta^{T-t}v_t)^2}+2D\left (\max_{t\in [1:T]} \lvert \beta^{T-t} v_t\rvert\right )\,.
\]
\end{theorem}

Now based on this discounted regret bound, we prove the claimed dynamic regret bound in \autoref{thm:dftrl-clip}. 
Similar to the proof of \autoref{thm:dftrl},
from  \autoref{thm:dftrl_formal}, we choose the partition to be $\cup_{t=1}^T\{t\}$, and let $U=D$ and $\bar u_t = u_t$ for all $t$. With such choices, we have
\begin{align}
\variation  \leq  \beta \sum_{t=1}^{T-1}
\abs{\sum_{s=1}^{t} \beta^{t-s} v_{s}} \abs{  u_{t} -   u_{t-1}}   \,.
\end{align}
Therefore, the upper bound in \autoref{thm:dftrl_formal} becomes   (notice that $\alpha=U=D$ and $\beta<1$)
\begin{align}
\regret_{T}(u_{0:T-1}) &\leq 2D\spar{\sqrt{V_\beta(v_{1:T})}   + (1-\beta) \sum_{t=1}^T  \sqrt{V_{\beta}(v_{1:t})}}+2DG\spar{1+(1-\beta)T} +\beta \sum_{t=1}^{T-1}
\abs{\sum_{s=1}^{t} \beta^{t-s} v_{s}} \abs{  u_{t} -   u_{t-1}} \\
&\leq  4D\spar{\sqrt{V_\beta(v_{1:T})}   + (1-\beta) \sum_{t=1}^T  \sqrt{V_{\beta}(v_{1:t})}}+   \sum_{t=1}^{T-1}
\abs{\sum_{s=1}^{t} \beta^{t-s} v_{s}} \abs{  u_{t} -   u_{t-1}} \,.
\end{align}
Now notice that for all $t$, since $\beta<1$,
\begin{align}
V_\beta(v_{1:t})&=\sum_{s=1}^{t} (\beta^{t-s}v_s)^2\leq G^2\sum_{i=0}^\infty \beta^{2i}=\frac{G^2}{1-\beta^2}<\frac{G^2}{1-\beta}\,, \quad \text{and}\\
\abs{\sum_{s=1}^{t} \beta^{t-s} v_{s}} &\leq G \sum_{t=0}^\infty \beta^t \leq \frac{G}{1-\beta}\,.
\end{align}
Substituting these bounds back to the bound on $R_T(u_{0:T-1})$, we obtain
\[
\regret_{T}(u_{0:T-1}) \leq  4DG\rpar{ \frac{1}{\sqrt{1-\beta}}  + \sqrt{1-\beta} \cdot T} +  \frac{GP}{1-\beta}\,.
\]
Therefore, we arrive at \autoref{thm:dftrl-clip} presented in the main paper.

\section{Benefits of momentum and discounting factor}

This section presents omitted details from \autoref{sec:components}. The goal is to justify the benefits of Adam's algorithmic components, namely the momentum and the discounting factor.

\subsection{Proof of lower bounds (\autoref{thm:lower})} \label{pf:lower}

For simplicity, we start by assuming $T$ is a multiple of $4$.
 Consider the following loss sequence:
 For $1\leq t\leq T/2$,
 \begin{align}
     \bv_t = \begin{cases}
         (1,0) &\text{for $t$ even,}\\
         (0,1) &\text{for $t$ odd.}
     \end{cases}
 \end{align}
For $ T/2< t \leq T$,
  \begin{align}
     \bv_t = \begin{cases}
         (-1,0) &\text{for $t$ even,}\\
         (0,-1) &\text{for $t$ odd.}
     \end{cases}
 \end{align}
The comparator sequence $\bu_{0:T-1}$ is given as $\bu_t = (-1,-1)$ for $0\leq t\leq T/2-1$ and $\bu_t = (1,1)$ for $t\geq T/2$.
Then, we have  $\sum_{t=1}^{T} \inp{\bv_{t}}{\bu_{t-1}} = -T$.

As for the total loss, 
\begin{itemize}
\item Consider the baseline \eqref{no-momentum}. Since $\bv_t[i]\bv_{t+1}[i] =0$ for all $t\geq 1$ and $i=1,2$, we have
\begin{align}
\sum_{t=1}^{T} \inp{\bv_{t}}{\update_{t-1}}&=\sum_{t=1}^T\sum_{i=1}^2\bv_t[i]\update_{t-1}[i]=-\sum_{t=1}^T\sum_{i=1}^2\balpha_{t-1}[i]\bv_{t-1}[i]\bv_t[i]=0.
\end{align}
\item Consider \ref{dftrl-clip} with $\beta=1$ and $D=1$. Recall its coordinate-wise update rule,
\begin{equation} 
\update_t[i] = -\clip_1\rpar{\alpha \frac{\sum_{s=1}^{t} \bv_{s}[i]}{ \sqrt{ \sum_{s=1}^t \bv_s[i]^2}}}\,.
\end{equation}
From the loss sequence, it follows that $\sum_{s=1}^{t} \bv_{s}[i]\geq 0$ for all $t$, and hence, we have $-1\leq \update_t[i]\leq 0$ for all $t\geq 0$ and $i=1,2$. Hence, $\sum_{t=1}^{T} \inp{\bv_{t}}{\update_{t-1}} \geq -T/2$.
\end{itemize}
This completes the proof under the assumption that $T$ is a multiple of $4$.

For general $T$, let $\hat T$ be the largest integer less or equal to $T$ which is a multiple of $4$. Then, we define $\bv_{1:\hat T}$ and $\bu_{1:\hat T-1}$ as the aforementioned loss and comparator sequences (with $T$ replaced by $\hat T$), and this yields lower bounds on $\sum_{t=1}^{\hat T} \inp{\bv_{t}}{\update_{t-1}-\bu_{t-1}}$. As for the time index satisfying $\hat T<t\leq T$, we define $\bv_t=(0,0)$ and $\bu_{t-1}=\bu_{\hat T-1}$. In this way, altogether, $\sum_{t=1}^{T} \inp{\bv_{t}}{\update_{t-1}-\bu_{t-1}}=\sum_{t=1}^{\hat T} \inp{\bv_{t}}{\update_{t-1}-\bu_{t-1}}$, and the lower bounds for the latter can be applied.

\subsection{Proof of \autoref{cor:discount_bound}}
\label{pf:cor:discount}

Using \autoref{thm:dftrl-clip} with $\beta = 1 - c T^{-\nicefrac{2}{3}}$,
\begin{align*}
R_T(u_{0:T-1})
&\lesssim \frac{DG}{\sqrt{1-\beta}} + \frac{DG}{1-\beta} + \sqrt{1-\beta} DGT \tag{\autoref{thm:dftrl-clip}}\\
&= \frac{DGT^{1/3}}{\sqrt c} + \frac{PGT^{2/3}}{c} + \sqrt{c} DGT^{2/3}\\
&\lesssim DGT^{2/3}c^{1/2} \Bigl(1 + \frac{c^{-3/2}P}D\Bigr)\,.
\end{align*}
With the optimal tuning $c = \Theta\bigl((P/D)^{\nicefrac{2}{3}}\bigr)$, it becomes $\mathcal{O}(GD^{\nicefrac{2}{3}}P^{\nicefrac{1}{3}}T^{\nicefrac{2}{3}})$.

\subsection{Proof of \autoref{cor:discount_unbound}}\label{pf:cor:discount_unbound}

Consider $\beta<1$ first. Since the environment is well-behaved with constant $M$, we can invoke \autoref{thm:dftrl} with $\cc$ there replaced by $M$. Notice that $M$ is independent of $\beta$, therefore at the end we may tune $\beta$ using $M$. Concretely, using \autoref{thm:dftrl} with $\beta=1-cT^{-1}$,
\begin{align}
R_T(u_{0:T-1})&\lesssim\frac{\rpar{ \alpha M^2+M P }G}{\sqrt{1-\beta}}+\sqrt{1-\beta}\alpha M^2GT\tag{\autoref{thm:dftrl}}\\
&=M G\sqrt{T}c^{\nicefrac 1 2}\rpar{\frac{\alpha M+P}{c}+\alpha M}\\
&\lesssim \alpha M^2 G\sqrt{T}c^{\nicefrac 1 2}\rpar{1+\frac{c^{-1}P}{\alpha M}}\,.
\end{align}
With the optimal tuning $c = \Theta\bigl(P/(\alpha M)\bigr)$, it becomes $\mathcal{O}(\alpha^{\nicefrac{1}{2}}M^{\nicefrac{3}{2}}GP^{\nicefrac{1}{2}}T^{\nicefrac{1}{2}})$.

Next, consider $\beta=1$. We follow the same analysis in \autoref{sec:simplify} until \eqref{eq:reduction_before_beta}, before plugging in any $\beta$. Then, instead of using $\beta<1$ there, we plug in $\beta=1$, which yields
\begin{align}
\regret_{[0,T-1]}(u_{0:T-1})& \leq \rpar{\frac{1}{2}\alpha M^2 +\sqrt{2}\alpha}\sqrt{V_1(v_{1:T})}+2\alpha M G + M \sum_{t=0}^{T-2} \sqrt{V_1(v_{1:t+1})}\abs{u_{t} - u_{t+1}}\\
& \lesssim \rpar{\alpha M^2 +\sqrt{2}\alpha}  G\sqrt{T} +  M  G\sqrt{T}  \sum_{t=0}^{T-2}\abs{u_{t} - u_{t+1}}\\
&\lesssim M GP\sqrt{T}\,.\tag{$T\gg 1$, and $P\gg \alpha M$}
\end{align}

\section{Details on optimization} \label{sec:proof}

\subsection{Proof of \autoref{thm:guarantee}}
\label{pf:thm:guarantee}
Since $F$ is differentiable, the fundamental theorem of calculus implies that 
for all  $\bx,\by\in \R^d$,  
$F(\by)-F(\bx) 
=  \int_0^1 \! \langle \nabla F(\bx+t(\by-\bx)), \by-\bx\rangle\, \mathrm{d}t$. 
Hence, we have 
\begin{align}
F(\bw_{t+1})-F(\bw_t) &= \int_0^1 \! \langle \nabla F(\bw_{t}+s \update_t),  \update_t\rangle\, \mathrm{d}s =\E_{s\sim \textup{Unif}([0,1]) } \inp{ \nabla F(\bw_{t}+s \update_t)}{\update_t}
\end{align}
Now, summing over $t$ and telescoping yield the desired equality.

\end{document}